\PassOptionsToPackage{table}{xcolor}
\documentclass[10pt,twocolumn,letterpaper]{article}

\usepackage[pagenumbers]{cvpr} 

\usepackage[ruled]{algorithm}
\usepackage[noend]{algorithmic}
\usepackage{makecell}
\usepackage{multirow}
\usepackage[table]{xcolor}

\usepackage{lipsum}
\usepackage[most]{tcolorbox}

\usepackage{graphicx}

\usepackage{amsthm}

\newtheorem{proposition}{Proposition}
\newtheorem*{proposition*}{Proposition}
\DeclareMathOperator*{\argmax}{arg\,max}

\renewcommand\algorithmiccomment[1]{%
  \hfill{\color{teal}$\triangleright$\ \small\textit{#1}}%
}

\definecolor{cvprblue}{rgb}{0.21,0.49,0.74}
\usepackage[pagebackref,breaklinks,colorlinks,allcolors=cvprblue]{hyperref}

\newcommand{\ours}{\textsc{VideoICL}\xspace}
\def\paperID{7088} 
\def\confName{CVPR}
\def\confYear{2025}

\title{VideoICL: Confidence-based Iterative In-context Learning\\ for Out-of-Distribution Video Understanding}

\author{
    Kangsan Kim$^{1*}$\; 
    Geon Park$^{1*}$\; 
    Youngwan Lee$^{1,3}$\;
    Woongyeong Yeo$^{1}$\;
    Sung Ju Hwang$^{1,2}$\\
    $^1$KAIST \;\; $^2$DeepAuto.ai \;\; $^3$ETRI \\
    \small{\texttt{\{kksan07, geon.park, ywlee88, wgcyeo, sjhwang82\}@kaist.ac.kr}} \;
}

\begin{document}
\maketitle
\def\thefootnote{*}\footnotetext{Equal contribution.}
\begin{abstract}

Recent advancements in video large multimodal models (LMMs) have significantly improved their video understanding and reasoning capabilities. However, their performance drops on out-of-distribution (OOD) tasks that are underrepresented in training data. Traditional methods like fine-tuning on OOD datasets are impractical due to high computational costs. While In-context learning (ICL) with demonstration examples has shown promising generalization performance in language tasks and image-language tasks without fine-tuning, applying ICL to video-language tasks faces challenges due to the limited context length in Video LMMs, as videos require longer token lengths. To address these issues, we propose VideoICL, a novel video in-context learning framework for OOD tasks that introduces a similarity-based relevant example selection strategy and a confidence-based iterative inference approach. This allows to select the most relevant examples and rank them based on similarity, to be used for inference. If the generated response has low confidence, our framework selects new examples and performs inference again, iteratively refining the results until a high-confidence response is obtained. This approach improves OOD video understanding performance by extending effective context length without incurring high costs. The experimental results on multiple benchmarks demonstrate significant performance gains, especially in domain-specific scenarios, laying the groundwork for broader video comprehension applications. Code will be released at \href{https://github.com/KangsanKim07/VideoICL}{https://github.com/KangsanKim07/VideoICL}

\end{abstract}
    
\section{Introduction}
\label{sec:intro}
\begin{figure*}
    \centering
    \includegraphics[width=\linewidth, trim={0cm 0.3cm 0cm 0.3cm}]{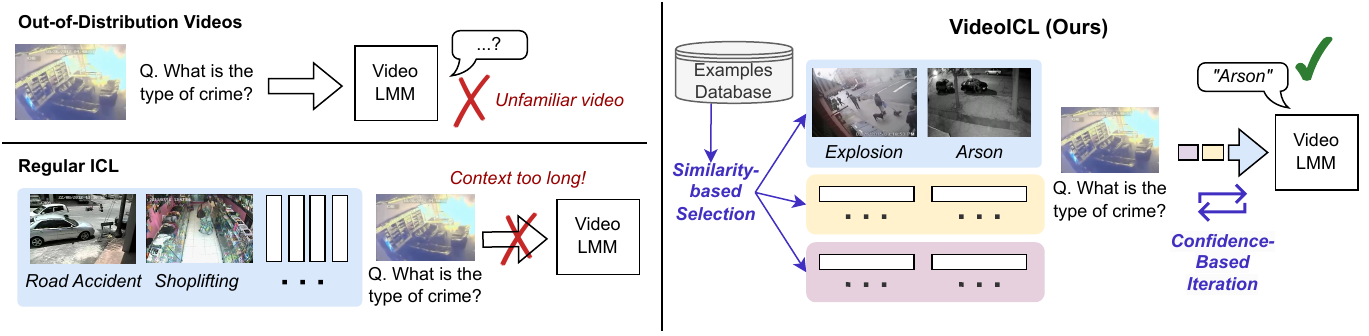}
    \caption{\textbf{Motivation.} \textbf{Top~left:} Video LMMs perform poorly in out-of-distribution videos, such as crime videos. \textbf{Bottom~left:} In-Context Learning (ICL), which is usually employed to solve this problem, is infeasible for video tasks, since the in-context demonstrations are too long. \textbf{Right:} \ours alleviates this problem by selecting the most relevant demonstrations~(\eg, 2-shot) by \textit{similarity-based example selection}, and iteratively performing inference with different sets of demonstrations at each step (\textit{confidence-based iterative inference}).}
    \label{fig:concept}
\end{figure*}
Recent video large multimodal models (LMMs)~\citep{zhang_video_2024, wang_qwen2-vl_2024, cheng2024videollama2advancingspatialtemporal, liu2024oryxmllmondemandspatialtemporal} have shown notable improvement in video understanding and reasoning tasks, \eg, enabling these models to comprehend natural scene videos and answer causal questions. However, as new types of data continue to emerge, these models are expected to handle previously unseen, out-of-distribution (OOD) videos~\citep{bashmal_capera_2023, sultani_real-world_2019, ng_animal_2022, li_sports-qa_2024, martin_driveact_2019, he2024pitvqaimagegroundedtextembedding}. Such OOD videos are rarely encountered during training due to their specialized nature or the need for domain-specific knowledge, such as gymnastic or surgical videos. Consequently, the performance of video LMMs on OOD videos is poor compared to in-distribution videos~\citep{marcu2024lingoqavisualquestionanswering, khattak2024goodvideolmmcomplex}, primarily due to the limited representation of OOD content in training datasets. For instance, video LMMs can readily distinguish between well-represented actions like ``dancing" and ``exercising" but struggle to differentiate actions like ``abuse" from ``assault", as crime videos are seldom included in training datasets. Given that the model faces numerous unseen situations in real-world scenarios, improving OOD video understanding with video LMMs remains a significant challenge.

Fine-tuning the model on the target videos is a straightforward method to enhance its performance. However, fine-tuning the model for each OOD scenario is time-consuming and often impractical, as it requires substantial amounts of training data to avoid overfitting and incurs significant training costs.
In contrast, in the realm of large language models~(LLMs), in-context learning (ICL)~\citep{brown2020icl}—which involves providing example inputs alongside a test sample during inference—has been actively explored as an efficient alternative for handling OOD tasks. ICL has shown strong generalization on unseen tasks in LLMs, making it advantageous as it bypasses the need for fine-tuning. Many studies have demonstrated the effectiveness of ICL in language-only and image-language tasks~\citep{han2024doesgpt4visionadaptdistribution, zhang2024outofdistributiongeneralizationmultimodallarge, yuan2023revisitingoutofdistributionrobustnessnlp, alayrac_flamingo_2022, li_otter_2023}; however, ICL's potential in video-language tasks has not been fully explored. 

A key challenge with ICL in the video domain is that video tokens are significantly longer than image or text tokens, limiting the number of video examples in a single context. For instance, LLaVA-Video~\citep{zhang_video_2024} can process up to 32K tokens~(\ie, context-length), but a 30-second, 384×384 resolution video sampled in 32 frames is converted to approximately 5.5K tokens through the vision encoder of the model with a patch size of 14. This allows for a maximum of only four video samples within the token limit. Considering that recent text-only and image-text ICL research is advancing toward many-shot learning~\citep{agarwal2024manyshotincontextlearning, jiang2024manyshotincontextlearningmultimodal, bertsch_-context_2024}, often utilizing over 1K examples, the number of video examples here is notably constrained in comparison.

This challenge has been approached from two main directions in multimodal ICL research. One approach aims to reduce the token length of each example~\citep{gao_aim_2024, zhuang2024vectoriclincontextlearningcontinuous}, while the other focuses on selecting a minimal set of highly effective examples~\citep{chen2023understandingimprovingincontextlearning, wang2024bayesianexampleselectionimproves}. However, the first approach potentially may lead to a loss of crucial information, while the second relies on a limited number of in-context examples, making the model performance highly sensitive to the relevance of these examples to the given query prompt. 

%

To address these issues, we propose \ours, a training-free video in-context learning framework for OOD video understanding, which efficiently handles multiple examples within a limited context length while preserving example quality. 
Specifically, we first introduce a similarity-based relevant example selection strategy, ranking demonstrating examples based on both video and text feature similarity to construct an ordered set of relevant examples for inference. 
Then, we present a confidence-based iterative inference mechanism;
\ours performs in-context learning iteratively, with each iteration leveraging a new subset of highly relevant examples from the top of the similarity ranking. 
At each iteration, the model calculates a confidence score based on token probabilities and stops iterating once it reaches a sufficient confidence level to generate an accurate answer.
This iterative approach enables the model to effectively utilize a larger pool of examples, maintaining informational richness without exceeding context limits.



Extensive experiments on several OOD video understanding benchmarks demonstrate the superiority of our method.
Specifically, we validate our approach across six datasets, spanning four distinct video-language tasks: multiple-choice question answering, open-ended question answering, video classification, and video captioning.
The results show that our framework, tested with LLaVA-Video-7B~\cite{zhang_video_2024}, outperforms zero-shot baselines, achieving an average improvement of 25.6\%p and up to 54.6\%p in QA and classification tasks, alongside a gain of 0.143 BLEU-4 points in video captioning. Notably, our approach using a 7B model outperforms a 72B model in zero-shot settings and on some datasets, \ours achieves better results than LoRA fine-tuned counterparts. These results underscore the effectiveness of our video ICL approach, showing that even smaller models can outperform larger and fine-tuned models on OOD tasks with a robust ICL framework.

Furthermore, \ours outperforms the state-of-the-art video ICL model, Otter~\citep{li_otter_2023}, which is pre-trained in an in-context manner, by a substantial margin. These results suggest that the training-free nature of \ours facilitates greater scalability and generalization to diverse OOD videos compared to Otter.
We believe that our findings will inspire further research into video in-context learning and drive advancements in enhancing the generalization capabilities of video LMMs for OOD videos.

\noindent Our main contributions are summarized below:
\begin{itemize}

    \item We introduce \ours, a novel training-free framework for video in-context learning that enhances out-of-distribution (OOD) video understanding without requiring model fine-tuning or compromising input quality.

    \item We propose a confidence-based iterative in-context learning approach that effectively leverages multiple examples, addressing the token length limitations of video LMMs.

    \item \ours achieves state-of-the-art results on six diverse OOD video-language datasets, with an average improvement of 25.6\%p and up to 54.6\%p in QA and classification tasks, along with a gain of 0.143 BLEU-4 points in video captioning, significantly outperforming zero-shot and baseline methods.


\end{itemize}


\section{Related works}
\label{sec:related_works}
\begin{figure*}[t]
    \centering
    \includegraphics[width=\linewidth, trim={0cm 1.0cm 0cm 0}]{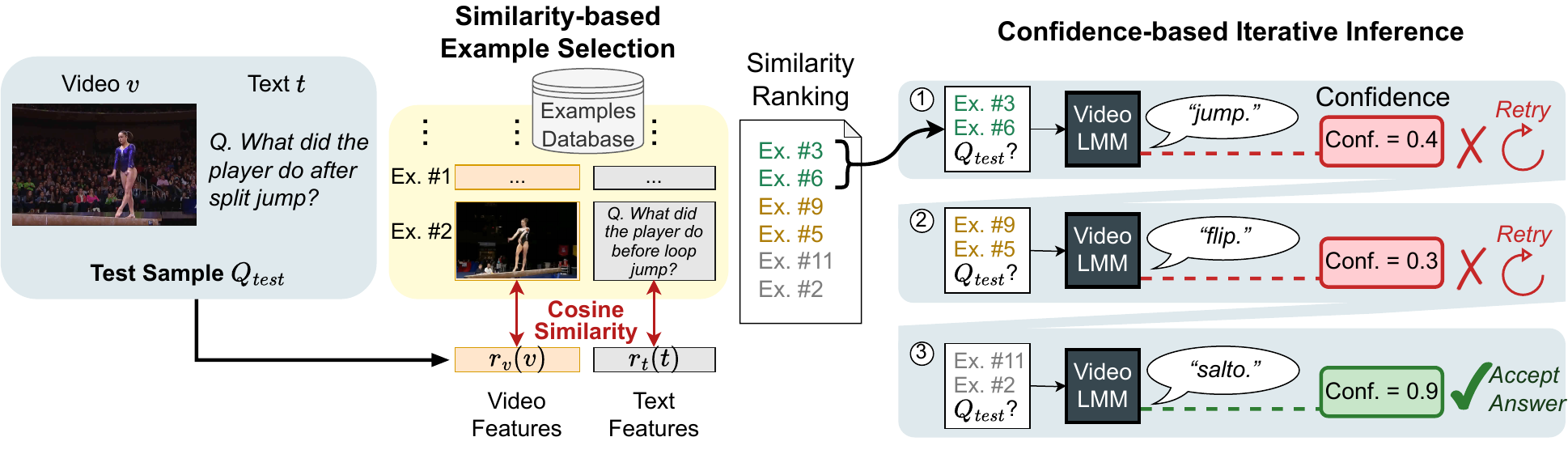}
    \caption{\textbf{Our Methodology.} Given a test query $Q_{test}$ consisting of a video and some text, each are embedded into a vector. \textbf{Similarity-based Example Selection:} Based on the cosine similarity between the query vector and the embeddings in the database of pre-encoded examples, we retrieve top-$k$ most similar examples. This stage takes negligible time cost since it only generates features from test samples and calculates the similarities with pre-encoded features. \textbf{Confidence-Based~Iterative Inference:} Starting from the top of the list, each set of $m$ examples are used as in-context examples for the query $Q_{test}$, until the confidence for the generated answer exceeds the threshold.}
    \label{fig:method}
\end{figure*}
\subsection{Video large multimodal models}

The impressive capabilities of LMMs have motivated substantial research into video LMMs~\citep{wang_qwen2-vl_2024, zhang_video_2024, cheng2024videollama2advancingspatialtemporal, liu2024oryxmllmondemandspatialtemporal}, demonstrating significant performance on video understanding benchmarks. However, the benchmarks commonly used for evaluating video LMMs, such as NExT-QA~\citep{xiao_next-qanext_2021}, Perception Test~\citep{patraucean_perception_2023}, and MVBench~\citep{li_mvbench_2023}, primarily evaluate models on general tasks that require only basic skills for typical videos. Therefore, the strong performance of video LMMs is largely limited to in-domain videos, which resemble the videos they were trained on.

On the other hand, studies have shown that video LMMs perform poorly when tested on OOD videos. \citet{khattak2024goodvideolmmcomplex} test different video LMMs using unusual and physically anomalous activities. 
Additionally, \citet{marcu2024lingoqavisualquestionanswering} reveal that video LMMs struggle with autonomous driving videos.
Given that the model encounters numerous unseen situations in real-world, these works underscore the need for further research to enhance OOD video understanding within video LMMs.


\subsection{Multimodal in-context learning}
Beyond the success of ICL on text-only tasks, it has been extended to multimodal tasks involving images~\citep{gao_aim_2024,sterner_few-shot_2024,xu_introspection_2024,li_otter_2023,zhao_mmicl_2024,chen_mmict_2024,chen2023understandingimprovingincontextlearning,baldassini_what_2024} and videos~\citep{yu_eliciting_2024,alayrac_flamingo_2022,li_otter_2023}. 
In video ICL, most studies aim to equip LMMs with the ability to comprehend multiple video examples by training them on video-text interleaved datasets. 
Otter~\citep{li_otter_2023} trains image-text ICL model using ICL instruction tuning dataset with video-language demonstrations. \citet{yu_eliciting_2024} propose a strategy to generate a video-text ICL training dataset that induces small-scale language models.
However, these studies do not address OOD videos, and fine-tuning is impractical for handling OOD tasks, as it is both costly and time-consuming to train the model for each task.

A significant challenge in multimodal ICL is the increased length of each demonstration. One approach to address this is to compress the tokens of context examples to fit within the model's limit, using fused virtual tokens~\citep{gao_aim_2024} or learnable parameters~\citep{chen_mmict_2024}.
However, such compression methods risk losing crucial details. Another approach is selecting the most effective demonstrations~\citep{li2023configuregoodincontextsequence, xu_introspection_2024, wang2024bayesianexampleselectionimproves}. 
MMICES~\citep{chen2023understandingimprovingincontextlearning} considers the varying importance of text and image information. 
However, since these methods still use a limited number of examples during inference, their performance can be less robust and highly sensitive to selection.


\subsection{Iterative in-context learning}
To enhance the robustness of ICL, some studies have employed iterative methods in language-only tasks. IDS~\citep{qin_ids_2024} utilizes iterative refinement by selecting multiple sets of demonstrations and performing majority voting among the generated answers to identify the best response. \(Se^2\)~\citep{liu_se2_2024} introduces an example selection approach that considers the sequence of examples, iteratively refining the selection to find the optimal set based on previously chosen examples. Iterative retrieval~\citep{chen2024learningretrieveiterativelyincontext} enhances ICL by using a stateful, policy-learning framework for iterative example selection, improving performance in semantic parsing tasks. Inspired by these works, we adopt an iterative method to address the challenge of limited example numbers in video ICL.



\subsection{Factual confidence estimation}
Estimating the factual confidence of LLMs has become an important research focus, as higher confidence levels are often associated with more accurate outputs and fewer hallucinations~\citep{guo2017calibrationmodernneuralnetworks, huang2023surveyhallucinationlargelanguage, feng2024donthallucinateabstainidentifying}. \citet{xiong2024can} and \citet{lin_teaching_2022} explore \textit{verbalization}, where models express confidence in natural language. \citet{kumar2024confidencehoodinvestigationconfidenceprobability} employ logit-based methods to estimate token-level confidence by assessing token probabilities as indicators, while \citet{liu_se2_2024} apply these techniques for selecting ICL examples. Sequence-level confidence is often derived by aggregating token-level confidence, typically using the minimum probability across all tokens~\citep{huang2023lookleapexploratorystudy}. \citet{azaria_internal_2023} and \citet{li2023inferencetime} utilize a \textit{trained probe} approach, using an MLP trained on supervised data to map hidden states to confidence values.

\section{Method}
\label{sec:method}
In this section, we describe our method to overcome the challenges of video ICL. First, we select a certain number ($k$) of in-context examples for a given query, constructing a list of demonstrating examples based on similarity (\cref{sec:sim_based_ex_selection}). Next, we sequentially fetch a small number ($m \leq k$) of examples from the list, iteratively refining the answer based on model confidence (\cref{sec:conf_based_rep}). The complete workflow of \ours is shown in \cref{fig:method}.

\subsection{Notation}
Throughout the paper, $\hat{y} = \mathsf{M}(x; \mathcal{D})$ denotes the answer a video LMM $\mathsf{M}$ generates for a query $x$ using a set of in-context examples $\mathcal{D}$. $\mathrm{Conf}_{\mathsf{M}}(x, \hat{y}; \mathcal{D})$ denotes the video LMM $\mathsf{M}$'s estimated confidence in a generated answer $\hat{y}$, given a question $x$ and a set of in-context examples $\mathcal{D}$.


\subsection{Similarity-based example selection}
\label{sec:sim_based_ex_selection}
At test time, we are given a query $x = (t, v)$, where $t$ and $v$ are the text and video in the query prompt, respectively. Based on the relevance to the query, $k$ in-context examples are selected from the set of target task-specific example data. $k$ is a fixed hyperparameter. For this, we use a linear combination of the cosine similarities on the vector representation of the given query and example data:
\begin{equation}
\begin{split}
    & S_Q\left(\left(t, v\right), \left(\tilde{t}, \tilde{v}\right)\right)  \\
    & := \alpha S_C(r_t(t), r_t(\tilde{t})) + (1 - \alpha) S_C(r_v(v), r_v(\tilde{v})),
\end{split}
\end{equation}
where $S_C$ denotes the cosine similarity between two vectors. $r_t(\cdot)$ and $r_v(\cdot)$ denote text and video encoders, respectively, which map arbitrary-length text and video inputs into fixed-length vectors. $\alpha$ is a balancing coefficient between the text and the video similarities.
From the set of example data $\mathsf{D} = \{(\tilde{t}_1, \tilde{v}_1), \dots, (\tilde{t}_n, \tilde{v}_n)\}$, we select the top-$k$ examples that maximize $S_Q$:
\begin{align}
     \mathrm{SelectRelevant}_k (\mathsf{D}, x) := \underset{\tilde{x} \in \mathsf{D}}{\mathrm{Top-}k} \left[ S_Q (x, \tilde{x}) \right].
\end{align}
Note that this example selection has negligible cost overhead at inference time, since all of the text and video vector representations in the example set can be pre-processed and be stored in a vector database.

\subsection{Confidence-based iterative inference}
\label{sec:conf_based_rep}
Since all $k$ in-context examples do not fit within the context size of open-source video LMMs, we instead provide the model $m$ examples at each iteration. $m \leq k$ is a fixed hyperparameter, chosen so that $m$ examples, along with the query, fit within the model's context length limit. Specifically, from the list of $k$ relevant examples ${\mathcal{D}}$, we first generate an answer for the given query $x$ using the first $m$ examples: $\hat{y}_1 = \mathsf{M}(x; \mathcal{D}_{1:m})$. If $\mathrm{Conf}_{\mathsf{M}}(x, \hat{y_1}; \mathcal{D}_{1:m}) > c_{th}$, \ie, the model's confidence for the generated answer exceeds a predefined confidence threshold $c_{th}$, then we output $\hat{y}_1$ as the final answer. Otherwise, we retry using the next $m$ examples $\mathcal{D}_{m+1:2m}$. We repeat this process until either the confidence for the last answer exceeds $c_{th}$ or all of $\mathcal{D}$ is exhausted. After multiple iterations are completed, we finally output the answer which had the highest confidence. 

For the confidence estimation, inspired by LLM literature~\cite{huang2023lookleapexploratorystudy, yang2023improving, mahaut_factual_2024} that often utilizes token probability, we also quantify the confidence score by taking the minimum token probabilities from the generated response.
Through empirical observations, we found that using the minimum token probability is a better measure of confidence. 
Given a generated sequence of logits, we first apply softmax normalization to convert logits to probabilities, obtaining a sequence of $\hat{y} = \{p_1, ..., p_T\}$. We then compute the confidence score as $c = \min_{i=1}^{T} p_i$,
where $p_i$ represents the probability of the $i$-th token in the generated sequence.
More confidence measures are explored in~\cref{subsec:conf}.

\begin{algorithm}[t]
\caption{\ours Inference Process}\label{alg:iteration}
\begin{algorithmic}[1]
\INPUT Query $x$, video LMM $\mathsf{M}$, example pool $\mathsf{D}$.
\OUTPUT Final answer $\hat{y}$.
\STATE $\mathcal{D} = \mathrm{SelectRelevant}_k (\mathsf{D}, x)$.
\STATE $i \leftarrow 1.$
\REPEAT
    \STATE $\mathcal{D}_{cur} \leftarrow \mathcal{D}_{(i-1)m+1:im}.$ \COMMENT{Select next batch of samples}
    \STATE $\hat{y}_i = \mathsf{M}(x; \mathcal{D}_{cur})$. \COMMENT{Generate answer}
    \STATE $c_i = \mathrm{Conf}_{\mathsf{M}}(x, \hat{y}_1; \mathcal{D}_{cur}).$ \COMMENT{Evaluate confidence}
    \STATE $i \leftarrow i + 1.$
\UNTIL{$c_i > c_{th}$~~or~~$im > k$.}
\STATE $\hat{y} = \hat{y}_{\argmax_i c_i}.$ \COMMENT{Select the largest confidence answer}
\end{algorithmic}
\end{algorithm}

\begin{figure*}[t]
    \centering
    \includegraphics[trim={0 .3cm 0 .4cm}, width=\linewidth]{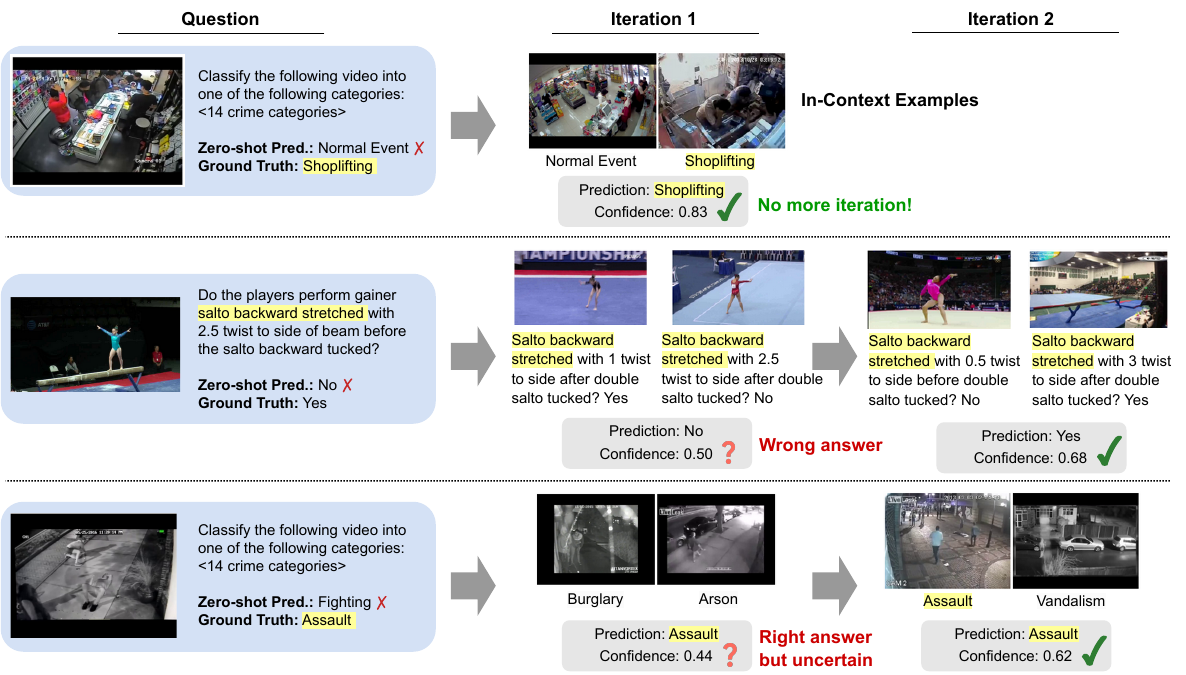}
    \caption{\textbf{Qualitative Results.} 
    We show three real hand-picked test samples from the main benchmarks. The first and third examples are from the UCF-Crime~\cite{sultani_real-world_2019} (video classification) task, and the second one is from the Sports-QA~\cite{li_sports-qa_2024} (open-ended QA) task. The leftmost column shows the given question, which the vanilla (non-ICL) model makes an incorrect prediction. The second and third columns show the first and second confidence-based iterations, with the selected in-context demonstrations at each iteration.
    }
    \label{fig:qualitative}
\end{figure*}

\subsection{Theoretical analysis}
To substantiate our approach, we further present a theoretical analysis illustrating how our confidence-based iteration method can enhance model accuracy.
Let us assume that the probability of a video LMM outputting the correct answer for a given task, given a query and a set of in-context examples, is constant: $\mathrm{Pr}(\mathsf{M}(x;\mathcal{D}_{\text{cur}}) = y) = p_c$, and independent across iterations.
For simplicity, we further assume that confidence estimation operates independently at each iteration and that the framework outputs the final iteration's answer as the final result, rather than selecting the answer with the highest confidence score.
\begin{proposition}[Asymptotic model accuracy]
    Let $a(n)$ be the expected accuracy of \ours with a maximum of $n$ confidence-based iterations. Then,
    \begin{align}
        \lim_{n \rightarrow \infty} a(n) = \frac{1}{1+\frac{\mathrm{FPR}}{\mathrm{TPR}} \cdot \frac{1-p_c}{p_c}},
    \end{align}
    where $\mathrm{TPR}$ and $\mathrm{FPR}$ stand for the true positive rate (i.e., recall) and the false positive rate of the confidence estimation method, respectively.
\end{proposition}
For example, if $p_c = 0.5$, $\mathrm{TPR} = 0.9$, and $\mathrm{FPR} = 0.1$, then the model's accuracy is expected to converge to 0.9 as the maximum number of iterations becomes large. In conclusion, if the confidence estimator is accurate enough, then our confidence-based iteration method is theoretically guaranteed to perform better than performing only one iteration. Please refer to \cref{sec:appendix_C} for the proof and rigorous statement of the proposition.

\begin{table*}[t]
    \newlength{\shiftdown}
    \setlength{\shiftdown}{-0.25em}
    \newcommand{\grey}[1]{\cellcolor[HTML]{eeeeee}{#1}}
    \newcommand{\vbar}[1]{%
    ~~
    }
    \newcommand{\ft}[1]{\textcolor{gray}{#1}}
    \centering
    \resizebox{\linewidth}{!}{
    \begin{tabular}{l|l|@{\hskip 6pt}lccccccccccccc@{\hskip 6pt}}
        \toprule
         \multicolumn{3}{c}{} & \multirow{2}{*}[-3em]{\makecell{$n$}} & \multirow{2}{*}[-3em]{\makecell{$k$}} & \makecell{Multiple\\Choice QA} & \multicolumn{2}{c}{Open-ended QA} & \multicolumn{2}{c}{\makecell{Video\\Classification}} & \multicolumn{6}{c}{Video Captioning} \\
         \cmidrule(l{2pt}r{2pt}){6-6} \cmidrule(l{2pt}r{2pt}){7-8} \cmidrule(l{2pt}r{2pt}){9-10} \cmidrule(l{2pt}r{2pt}){11-16}
         \multicolumn{3}{c}{} &&& \multirow{2}{*}[\shiftdown]{\makecell{Animal\\Kingdom}} & \multirow{2}{*}[\shiftdown]{\makecell{Sports-\\QA}} & \multirow{2}{*}[\shiftdown]{\makecell{Pit-\\VQA}} & \multirow{2}{*}[\shiftdown]{\makecell{UCF-\\Crime}} & \multirow{2}{*}[\shiftdown]{\makecell{Drive\\\&Act}} & \multicolumn{6}{c}{CapERA} \\
         \cmidrule(l{2pt}r{2pt}){11-16}
         \multicolumn{3}{c}{} &&&&&&&& {\scriptsize BLEU-1} & {\scriptsize BLEU-2}& {\scriptsize BLEU-3}& {\scriptsize BLEU-4}& {\scriptsize METEOR}& {\scriptsize ROUGE-L}\\
         \midrule
         \multicolumn{3}{l}{GPT-4o~\citep{openai_gpt-4o_2024}} & - & 0 & 58.2 & - & 6.9 & 58.0 & - & 0.143 & 0.065 & 0.037 & 0.023 & 0.142 & 0.173 \\
         \multicolumn{3}{l}{Gemini-1.5 Pro~\citep{geminiteam2024gemini15unlockingmultimodal}} & - & 0 & 72.9 & - & 14.7 & 55.1 & - & 0.126 & 0.057 & 0.031 & 0.019 & 0.134 & 0.176 \\
         \multicolumn{3}{l}{Otter-7B~\citep{li_otter_2023}} & 1 & 8 & 19.4 & - & 21.8 & 6.8 & - & 0.241 & 0.135 & 0.088 & 0.059 & 0.169 & 0.167 \\
         \midrule
         \multirow{9}{*}{\makebox[0.3em][c]{\rotatebox{90}{LLaVA-Video~\citep{zhang_video_2024}}}} & \makecell[l]{72B} & 
         {Zero-shot} & - & 0 & 69.7 & 25.7 & 5.7 & 35.6 & 14.6 & 0.133 & 0.060 & 0.034 & 0.020 & 0.129 & 0.170 \\
         \cmidrule(l{2pt}r{2pt}){2-16}
         &
         \makecell[l]{7B} & \makecell[l]{\ft{LoRA FT}} & \ft{-} & \ft{0} & \ft{70.2} & \ft{-} & \ft{40.5} & \ft{51.9} & \ft{-} & \ft{0.528} & \ft{0.393} & \ft{0.302} & \ft{0.227} & \ft{0.271} & \ft{0.181} \\
         && {Zero-shot} & - & 0 & 68.0 & 25.5 & 6.7 & 39.3 & 20.2 & 0.162 & 0.077 & 0.045 & 0.027 & 0.149 & \textbf{0.181} \\
         && \makecell[l]{MMICES~\citep{chen2023understandingimprovingincontextlearning}} & 1 & 2 & 69.3 & 43.0 & 46.4 & 50.7 & 51.3 & 0.462 & 0.312 & 0.224 & 0.160 & 0.245 & 0.178 \\
         && \textsc{SimRankOnce} & 1 & 2 & 69.3 & 41.8 & 54.0 & 50.7 & 52.0 & 0.462 & 0.312 & 0.224 & 0.160 & 0.245 & 0.178 \\
         && \textsc{RandExVote} & 4 & 8 & 69.6 & 21.5 & 11.5 & 36.6 & 19.9 & 0.418 & 0.256 & 0.170 & 0.116 & 0.189 & 0.153 \\
         && \textsc{SimRankVote} & 4 & 8 & 70.9 & 36.3 & 57.6 & 50.6 & 50.6 & 0.464 & 0.314 & 0.228 & 0.165 & 0.242 & 0.175 \\
         && \textbf{\ours (Ours)} & 4 & 8 & \textbf{72.3} & \textbf{47.6} & \textbf{61.3} & \textbf{53.3} & \textbf{53.4} & \textbf{{0.465}} & \textbf{{0.320}} & \textbf{{0.235}} & \textbf{{0.170}} & \textbf{{0.252}} & 0.178 \\
         &&\multicolumn{1}{l}{\color{NavyBlue}\grey{~$\Delta$}} & \grey{} & \grey{} & \color{NavyBlue}\grey{+4.3} & \color{NavyBlue}\grey{+22.1} & \color{NavyBlue}\grey{+54.6} & \color{NavyBlue}\grey{+14.0} & \color{NavyBlue}\grey{+33.2} & \color{NavyBlue}\grey{+0.302} & \color{NavyBlue}\grey{+0.242} & \color{NavyBlue}\grey{+0.190} & \color{NavyBlue}\grey{+0.143} & \color{NavyBlue}\grey{+0.104} & \grey{-0.003}\\
         \midrule
         \multirow{6}{*}{\makebox[0.3em][c]{\rotatebox{90}{Qwen2-VL~\citep{wang_qwen2-vl_2024}}}} & 7B & 
         {Zero-shot} & - & 0 & 58.6 & 26.8 & 5.8 & 36.1 & 10.6 & 0.286 & 0.159 & 0.101 & 0.066 & 0.149 & 0.138 \\
         && \textsc{SimRankOnce} & 1 & 2 & 63.8 & 43.2 & 55.3 & 46.3 & 45.4 & 0.457 & 0.310 & 0.223 & 0.158 & 0.249 & 0.189 \\
         && \textsc{RandExVote} & 4 & 8 & 62.3 & 21.0 & 14.0 & 36.6 & 14.3 & 0.397 & 0.240 & 0.157 & 0.104 & 0.188 & 0.170 \\
         && \textsc{SimRankVote} & 4 & 8 & 64.0 & 50.9 & 59.4 & 46.7 & 45.8 & 0.448 & 0.300 & 0.213 & 0.151 & 0.239 & 0.187 \\
         && {\textbf{\ours (Ours)}} & 4 & 8 & \textbf{66.3} & \textbf{51.5} & \textbf{59.6} & \textbf{48.7} & \textbf{49.3} & \textbf{0.471} & \textbf{0.329} & \textbf{0.244} & \textbf{0.176} & \textbf{0.265} & \textbf{0.189} \\
         && \multicolumn{1}{l}{\color{NavyBlue}\grey{~$\Delta$}} & \grey{} & \grey{} & \color{NavyBlue}\grey{+7.7} & \color{NavyBlue}\grey{+24.7} & \color{NavyBlue}\grey{+53.8} & \color{NavyBlue}\grey{+12.6} & \color{NavyBlue}\grey{+38.7} & \color{NavyBlue}\grey{+0.185} & \color{NavyBlue}\grey{+0.170} & \color{NavyBlue}\grey{+0.143}  & \color{NavyBlue}\grey{+0.110} & \color{NavyBlue}\grey{+0.116} & \color{NavyBlue}\grey{+0.051}\\
         \midrule
         \multirow{3}{*}{\makebox[0.3em][c]{\rotatebox{90}{\makecell{Oryx-1.5}}}} & \makecell{7B} & 
         {Zero-shot} & - & 0 & 58.6 & 28.3 & 3.8 & 11.9 & 10.7 & 0.242 & 0.126 & 0.077 & 0.049 & 0.140 & 0.151 \\
         & \citep{liu2024oryxmllmondemandspatialtemporal} & {\textbf{\ours (Ours)}} & 4 & 8 & 58.5 & \textbf{52.0} & \textbf{58.4} & \textbf{44.0} & \textbf{57.3} & \textbf{0.327} & \textbf{0.195} & \textbf{0.128} & \textbf{0.086} & \textbf{0.188} & \textbf{0.179} \\
         && \multicolumn{1}{l}{\color{NavyBlue}\grey{~$\Delta$}} & \grey{} & \grey{} & \grey{-0.1} & \color{NavyBlue}\grey{+23.7} & \color{NavyBlue}\grey{+54.6} & \color{NavyBlue}\grey{+32.1} & \color{NavyBlue}\grey{+46.6} & \color{NavyBlue}\grey{+0.085} & \color{NavyBlue}\grey{+0.069} & \color{NavyBlue}\grey{+0.052} & \color{NavyBlue}\grey{+0.038} & \color{NavyBlue}\grey{+0.047} & {\color{NavyBlue}\grey{+0.028}} \\
         \bottomrule
    \end{tabular}
    }
    \vspace{-0.2cm}
    \caption{\textbf{Main Results.} We evaluate our framework \ours on a wide range of \textit{out-of-distribution} benchmarks. $n= k/m$ denotes the (maximum) number of iterations, where $k$ is the total number of in-context examples, and $m$ is the number of examples used in each iteration. The difference ($\Delta$) denotes the score improvement of \ours over the vanilla 7B models without ICL. The LoRA FT baseline, fine-tuned on each of the downstream tasks, is shown in gray because it is not directly comparable to our method due to its additional training cost. The highest scores for each benchmark are shown in bold.}
    \label{tab:main}
\end{table*}
\section{Experiments}
\label{sec:experiments}
In this section, we validate our framework in four video-language tasks on a wide range of \textit{out-of-distribution} benchmarks using six datasets: multiple-choice QA on Animal Kingdom~\cite{ng_animal_2022}, open-ended QA on Sports-QA~\cite{li_sports-qa_2024} and PitVQA~\cite{he2024pitvqaimagegroundedtextembedding}, video classification on UCF-Crime~\cite{sultani_real-world_2019} and Drive\&Act~\cite{martin_driveact_2019}, and video captioning task on CapERA~\cite{bashmal_capera_2023}.

\subsection{Experiment setup}
We apply our framework to state-of-the-art open-source video LMMs, such as LLaVA-Video-7B~\citep{zhang_video_2024}, Qwen2-VL-7B~\citep{wang_qwen2-vl_2024}, and Oryx-1.5-7B~\citep{liu2024oryxmllmondemandspatialtemporal}, selected for their strong performance in general video understanding benchmarks.
For the similarity-based example selection, we use SentenceBERT~\citep{reimers_sentence-bert_2019} and InternVideo2~\citep{wang_internvideo2_2024} as text and video encoders, respectively. We sample each video at 1 frame per second across benchmarks, and for videos over 32 seconds, we uniformly select 32 frames from the entire sequence. For the confidence estimation, we use minimum value among the probabilities of generated tokens, and the confidence threshold is set to ${c_{th}=0.7}$ in multiple-choice QA and ${c_{th}=0.5}$ in other tasks. Our reasoning for this choice and additional experiments with various ${c_{th}}$ settings are provided in \cref{sec:rationale}. The total number of demonstrating examples is set to ${k=8}$. Note that considering the limited context length of the video LMMs, we set the number of in-context examples to ${m=2}$ at each iteration across all benchmarks. Therefore, the maximum number of iterations per test sample is ${n := k/m = 4}$. 

\subsection{Baselines}
We compare \ours with several strong baselines to rigorously assess our method's performance: GPT-4o~\citep{openai_gpt-4o_2024}, Gemini-1.5 Pro~\citep{geminiteam2024gemini15unlockingmultimodal}, Otter-7B~\citep{li_otter_2023}, MMICES~\citep{chen2023understandingimprovingincontextlearning}, Zero-shot, \textsc{SimRankOnce}, \textsc{RandExVote}, and \textsc{SimRankVote}.
Otter~\citep{li_otter_2023} is a large multimodal model fine-tuned on MIMIC-IT~\citep{li_mimic-it_2023}, which is a training set for image and video ICL. We exclude Flamingo~\citep{alayrac_flamingo_2022} and EILeV~\citep{yu_eliciting_2024} from the baselines because Flamingo does not provide a checkpoint, and EILeV is specifically trained on egocentric videos for the video captioning task.
MMICES~\citep{chen2023understandingimprovingincontextlearning} is an example selection method originally designed for image ICL, which we have extended to our video multimodal setting.
In the zero-shot baseline, responses are generated without any in-context examples.

To further evaluate the impact of each component within \ours, we use three additional baseline methods. \textsc{SimRankOnce} performs ICL a single time with only $m$ relevant examples chosen by our similarity ranking (as detailed in \cref{sec:sim_based_ex_selection}), without any iterative inference. \textsc{RandExVote} incorporates iterative inference but omits confidence-based selection, instead employing majority voting across answers generated using randomly selected in-context examples. \textsc{SimRankVote} also performs majority voting~(\ie, without confidence) but uses $k$ relevant examples chosen by our similarity ranking instead of random examples.
We also compare against the zero-shot performance of a larger 72B model.
For completeness, we include the performance of LLaVA-Video-7B fine-tuned with LoRA~\citep{hu_lora_2021} for some of the benchmark tasks. However, it is not directly comparable to ICL, since fine-tuning incurs a large training cost for each OOD task. 

\subsection{Datasets}
We test our method on six diverse datasets across four tasks, each focusing on specialized domains that are rarely covered in the training data of current video LMMs. For the multiple-choice QA task, we use the Animal Kingdom dataset~\citep{ng_animal_2022}, which includes animal videos annotated with actions across 140 different classes. For open-ended question-answering, we apply Sports-QA~\citep{li_sports-qa_2024}, covering various sports videos, and PitVQA~\citep{he2024pitvqaimagegroundedtextembedding}, designed specifically for visual QA in endonasal pituitary surgery videos. For open-ended QA, we utilize UCF-Crime~\citep{sultani_real-world_2019}, which categorizes types of crime in security camera footage into 13 groups, and Drive\&Act~\citep{martin_driveact_2019}, which identifies 34 types of activities performed by drivers in Kinect-IR videos. For the video captioning task, we evaluate models on the CapERA dataset~\citep{bashmal_capera_2023}, which is tailored to describe scenes from an aerial view. For more details, please see \cref{sec:appendix_B}.

\subsection{Quantitative results}
\cref{tab:main} summarizes the performance comparison with the baselines.
Overall, \ours outperforms the baselines across a wide variety of domains. See \cref{sec:additional_discussion} for a more detailed discussion on the baseline methods' performance.

\par\noindent\textbf{Multiple choice QA.}\quad
We observe that \ours achieves a +4.3\%p improvement in accuracy for recognizing animal actions compared to zero-shot LLaVA-Video-7B. Notably, \ours even surpasses the larger LLaVA-Video-72B model, despite using the smaller LLaVA-Video-7B model. This result indicates that simply increasing model size is not an effective solution for OOD video understanding. Furthermore, video ICL baselines~(\eg, \textsc{SimRankVote}) that rely solely on similarity ranking without confidence-based iteration perform less effectively than \ours, underscoring the advantages of our approach. 

\par\noindent\textbf{Open-ended QA.}\quad
We observe the highest accuracy improvement in open-ended QA tasks, with \ours achieving up to +54.6\%p and +22.1\%p improvements on PitVQA and Sports-QA, respectively, compared to zero-shot performance. This result highlights that ICL is particularly effective for out-of-distribution tasks where answers adhere to a specific format, even when the question itself is phrased simply. For example, given the question “\textit{What is the athlete doing?}” with a video showing a gymnast performing a move called \textit{salto backward tucked}, the model might respond with, “\textit{The athlete is jumping from the ground in a crowded gym.}” While this answer is not incorrect, it lacks specific domain knowledge users might expect in the context of the gymnastics domain, such as “\textit{The athlete is performing salto backward tucked}.” ICL enables the model to align with the expected answer style and apply relevant domain knowledge, yielding a response that is both more precise and contextually accurate.

\begin{table}[t!]
    \setlength{\shiftdown}{-0.3em}
    \newcommand{\grey}[1]{\cellcolor[HTML]{eeeeee}{#1}}
    \centering
    \resizebox{\linewidth}{!}{
    \begin{tabular}{lcccccc}
        \toprule
         & \makecell{\small Animal Kingdom} & \makecell{PitVQA} & \makecell{UCF-Crime} \\
         \midrule
         \makecell[l]{Baseline} & 68.0 & 6.7 & 39.3 \\
         \midrule
         Random & 68.4 \textcolor{NavyBlue}{(+0.4)} & 8.3 \textcolor{NavyBlue}{(+1.6)} & 38.4 \textcolor{NavyBlue}{(-0.9)} \\
         Text only & - & 33.1 \textcolor{NavyBlue}{(+24.8)} & - \\
         Video only & - & 29.1 \textcolor{NavyBlue}{(+22.4)} & - \\
         Text + Video & \textbf{72.3} \textcolor{NavyBlue}{(+4.3)} & \textbf{61.3} \textcolor{NavyBlue}{(+54.6)} & \textbf{53.3} \textcolor{NavyBlue}{(+14.0)} \\
         \bottomrule
    \end{tabular}
    }
    \caption{\textbf{Ablation on the similarity-based example selection method.} 
    We investigate which feature types are more helpful in selecting relevant examples.
    }
    \label{tab:ablation_demselect}
\end{table}

\par\noindent\textbf{Video Classification.}\quad
We observe that \ours significantly outperforms baselines by up to +14.0\%p more accuracy gain in UCF-Crime and +38.7\%p in Drive\&Act, demonstrating the effectiveness of our method in video classification tasks. Surprisingly, \ours exceeds the LoRA fine-tuned model on UCF-Crime. As \citet{bertsch_-context_2024} discovered, \ours tends to surpass fine-tuning in a small-size dataset. These outstanding results may be due to the fact that similar in-context examples also have a high chance of being from the same class, demonstrating the usefulness of our similarity-based example selection process.

\par\noindent\textbf{Video Captioning.}\quad
Across various metrics, our method consistently surpasses the baselines, demonstrating its effectiveness in video captioning tasks. The largest improvement was observed in BLEU-1, with the gains gradually decreasing from BLEU-2 to BLEU-4. This suggests that the examples in \ours provide useful words related to the query video, enhancing the relevance of generated captions.

\par\noindent\textbf{Comparison to proprietary models.}\quad
In \cref{tab:main}, we compare the performance of \ours with two leading proprietary APIs: GPT-4o~\citep{openai_gpt-4o_2024} and Gemini-1.5 Pro~\citep{geminiteam2024gemini15unlockingmultimodal}. Due to cost constraints, we evaluate one benchmark per task format. For this analysis, we use the results of \ours with the LLaVA-Video-7B model. While Gemini-1.5 Pro outperforms \ours by +0.6\%p in the Animal Kingdom benchmark, it falls short on the PitVQA and CapERA benchmarks. Similarly, GPT-4o achieves a +4.7\%p advantage over \ours in the UCF-Crime benchmark but fails on the Animal Kingdom, PitVQA, and CapERA tasks. Overall, \ours demonstrates an average performance improvement of +17.8\%p over GPT-4o and +14.2\%p over Gemini-1.5 Pro. These results highlight the effectiveness and robustness of \ours across diverse OOD videos, even when compared to significantly larger models. Importantly, \ours achieves this performance with a 7B-parameter backbone, while GPT-4o and Gemini-1.5 Pro are much larger models. Despite this inherent size disadvantage, \ours excels due to its efficient iterative in-context learning approach, underscoring its strong capabilities in challenging scenarios.

\subsection{Qualitative results}
We illustrate some representative samples from our benchmarks in \cref{fig:qualitative}. In the first row, we show an example from UCF-Crime~\citep{sultani_real-world_2019}. While the ground truth label is \textit{Shoplifting}, the zero-shot model answers incorrectly as \textit{Normal Event}. In contrast, our model selects two relevant demonstrations—a \textit{Normal Event} and a \textit{Shoplifting} example in the first round. Using these two in-context examples, our model is able to answer the original question correctly. This demonstrates that our similarity-based example selection allows the model to select relevant and helpful examples.

In the second row, we show an example from Sports-QA~\cite{li_sports-qa_2024} where the vanilla model makes an incorrect prediction. In the first iteration, the initial response of our model is incorrect too, with low confidence. By the second iteration, our model could answer correctly with high confidence. Without the iterative approach, the model would not have reached the correct answer, highlighting the effectiveness of iterative selection, especially when initial demonstrations are suboptimal.

In contrast, in the third row, even though the model predicts the right answer in the first iteration, it lacks confidence due to insufficiently relevant demonstrations, as there is no example from the same class as the ground truth. Only after observing relevant examples in the second iteration does it gain enough confidence in the correct answer, suggesting that its initial correct prediction was a fluke. We provide more qualitative results in \cref{sec:additional_qualitative}.

\section{Analysis}
\label{sec:analysis}

\begin{table}[t!]
    \setlength{\shiftdown}{-0.3em}
    \newcommand{\grey}[1]{\cellcolor[HTML]{eeeeee}\textcolor{NavyBlue}{#1}}
    \centering
    \resizebox{0.95\linewidth}{!}{
    \renewcommand{\arraystretch}{1.1}
    \renewcommand{\tabcolsep}{8pt}
    \begin{tabular}{lcccccccccc}
        \toprule
         & \multirow{2}{*}[\shiftdown]{\makecell{Animal\\\small Kingdom}} & \multirow{2}{*}[\shiftdown]{\makecell{Pit-\\VQA}} & \multirow{2}{*}[\shiftdown]{\makecell{UCF-\\Crime}} & \multicolumn{2}{c}{CapERA} \\
         \cmidrule(l{2pt}r{2pt}){5-6}
         &&&& {\scriptsize BLEU-4} & {\scriptsize METEOR} \\
         \midrule
         \makecell[l]{Baseline} & 68.0 & 6.7 & 39.3 &  0.027 & 0.149\\
         \midrule
         $k=2$ & 69.3 & 54.0 & 50.7 & 0.160 & 0.245 \\
         $k=4$ & 71.0 & 59.5 & 52.7 & 0.168 & 0.251  \\
         $k=8$ & 72.3 & \textbf{61.3} & 53.3 & \textbf{0.170} & \textbf{0.253} \\
         \grey{~$\Delta$} & \grey{+4.3} & \grey{+54.6} & \grey{+14.0} &  \grey{+0.143} & \grey{+0.104} \\
         $k=16$ & \textbf{73.2} & 61.2 & \textbf{53.6} & 0.169 & 0.250 \\
         \grey{~$\Delta$} & \grey{+5.2} & \grey{+54.5} & \grey{+14.3} & \grey{+0.142} & \grey{+0.101} \\
         \bottomrule
    \end{tabular}
    }
    \caption{\textbf{Ablation on the total number of demonstrations ($k$).} The highest scores for each benchmark are shown in bold.}
    \label{tab:ablation_niters}
\end{table}

In this section, we analyze the impact of each component and hyperparameter of our method on its performance. For our analysis, we select one representative dataset for each video-language task: Animal Kingdom~\citep{ng_animal_2022}, PitVQA~\citep{he2024pitvqaimagegroundedtextembedding}, UCF-Crime~\citep{sultani_real-world_2019}, and CapERA~\citep{bashmal_capera_2023}.


\subsection{Demonstration selection method}
\noindent The impact of similarity-based example selection is studied by analyzing which query features (text or video) most influence the effectiveness of example selection.
We compare our method with random demonstration selection, text feature-only selection, and video feature-only selection across three datasets. The results, shown in~\cref{tab:ablation_demselect}, indicate that our similarity-based selection method using both text and video features outperforms the random selection baseline by a large margin. For instance, on PitVQA, similarity-based selection leads to up to +53.0\%p performance increase over random selection. Furthermore, we observe that both text and video features contribute to similarity-based selection, as using only one of the two yields lower scores.

\subsection{Total number of available examples (\texorpdfstring{$k$}{k})}
\noindent We investigate the impact of confidence-based iterative inference by comparing our method, which uses multiple iterations, with a single-iteration baseline across four datasets. The results, shown in \cref{tab:ablation_niters}, indicate that increasing the total number of demonstrations generally helps the benchmark performance. For example, compared to performing only one iteration ($k$ = 2), we observe up to +7.2\%p performance increase in PitVQA when using eight iterations ($k$ = 16). This is consistent with the general observation that more ICL demonstrations lead to higher performance in \citet{bertsch_-context_2024}. This demonstrates that our confidence-based iterative inference can serve as an alternative to using many demonstrations at once when the context length is limited in a video ICL setting.

\begin{table}[t!]
    \setlength{\shiftdown}{-0.3em}
    \newcommand{\grey}[1]{\cellcolor[HTML]{eeeeee}\textcolor{NavyBlue}{#1}}
    \centering
    \resizebox{0.95\linewidth}{!}{
    \renewcommand{\arraystretch}{1.15}
    \begin{tabular}{lccccccccccc}
        \toprule
         & \multirow{2}{*}[\shiftdown]{\makecell{Animal\\\small Kingdom}} & \multirow{2}{*}[\shiftdown]{\makecell{Pit-\\VQA}} & \multirow{2}{*}[\shiftdown]{\makecell{UCF-\\Crime}} & \multicolumn{2}{c}{CapERA} \\
         \cmidrule(l{2pt}r{2pt}){5-6}
         &&&& {\scriptsize BLEU-4} & {\scriptsize METEOR}  \\
         \midrule
         \makecell[l]{Baseline} & 68.0 & 6.7 & 39.3 & 0.027 & 0.149 \\
         \midrule
         Verbalization & 69.7 & 54.6 & 51.8 & 0.160 & 0.245 \\
         \grey{~$\Delta$} & \grey{+1.7} & \grey{+47.9} & \grey{+12.5} & \grey{+0.133} & \grey{+0.096} \\
         Trained Probe & 71.7 & 42.5 & 52.7 & 0.162 & 0.250 \\
         \grey{~$\Delta$} & \grey{+3.7} & \grey{+35.8} & \grey{+13.4} & \grey{+0.135} & \grey{+0.101} \\
         Token Prob. & \textbf{72.3} & \textbf{61.3} & \textbf{53.3} & \textbf{0.170} & \textbf{0.253} \\
         \grey{~$\Delta$} & \grey{+4.3} & \grey{+54.6} & \grey{+14.0} & \grey{+0.143} & \grey{+0.104} \\
         \bottomrule
    \end{tabular}
    }
    \caption{\textbf{Ablation on the confidence estimation method.} `Token Prob.' refers to Token Probability. The highest scores for each benchmark are shown in bold.}
    \label{tab:ablation_confestim}
\end{table}
\subsection{Confidence estimation method}\label{subsec:conf}
\noindent We evaluate our confidence estimation method, token probability~\citep{huang2023lookleapexploratorystudy}, against two alternative approaches: \textit{verbalization}~\citep{xiong2024can, lin_teaching_2022} and \textit{trained probe}~\citep{azaria_internal_2023, kadavath2022languagemodelsmostlyknow}, across four benchmark datasets. For the trained probe, following \citet{azaria_internal_2023}, we pre-train a 4-layer MLP as a confidence estimator using the hidden states of the last token on a subset of diverse video-language datasets and benchmarks~\citep{zhang_video_2024, li2023seedbenchbenchmarkingmultimodalllms, fu2024videommefirstevercomprehensiveevaluation, li_mvbench_2023}. The confidence score from the trained probe is used in the same way as our main method. For verbalization, we ask the model if it is confident enough to respond before generating the answer. If it answers yes, we proceed with the response; if not, we iterate with the next set of examples. The results are shown in \cref{tab:ablation_confestim}.

The token probability method outperforms other approaches, with the trained probe following closely, while verbalization performs the worst. This somewhat contradicts the findings in~\cite{mahaut_factual_2024}, which report that the trained probe estimates model confidence most accurately. A possible explanation is that, unlike other zero-shot methods, the trained probe’s reliance on pre-training makes it more dependent on its specific training data, limiting its ability to generalize, potentially reducing robustness on OOD tasks~\citep{orgad2024llmsknowshowintrinsic}.

\subsection{Most confident responses across iterations}
\begin{figure}[t]
    \centering
    \includegraphics[width=\linewidth,trim={.5cm .2cm .5cm .3cm}]{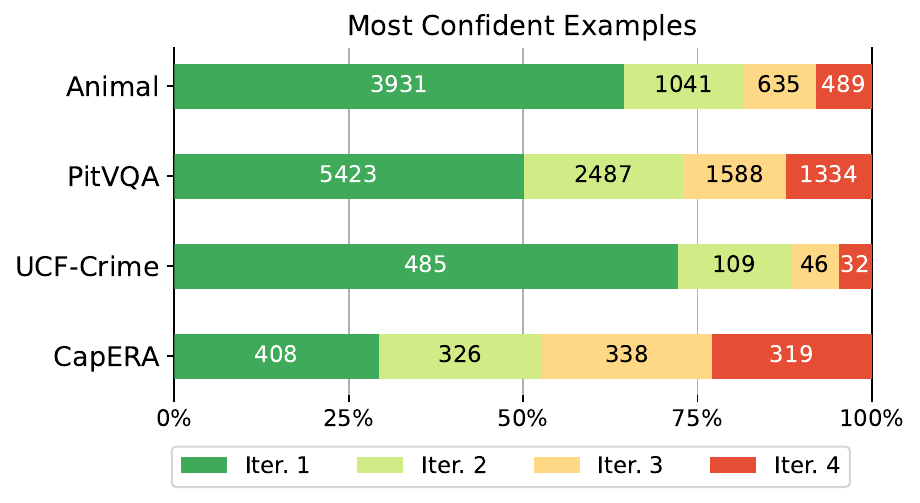}
    \vspace{-0.5cm}
    \caption{\textbf{Most confident examples.} The numbers on each bar represent the number of test samples where the corresponding iteration ended up having the highest confidence score. The x-axis represents the proportion of each iteration.}
    \label{fig:confichart}
\end{figure}
We examine which iteration yields the highest confidence scores. \cref{fig:confichart} illustrates the confidence across four datasets and highlights the iteration, out of four, that produced the most confident answers. Notably, despite using the most similar examples in the first iteration, only 29\% of responses in CapERA and 50\% in PitVQA achieved enough confidence to surpass the threshold in the first round. This finding suggests that relying only on the first iteration may be inadequate. Multiple iterations allow for adding more examples, which may boost confidence in later rounds.
\section{Conclusion}
In this paper, we propose \ours, a novel framework for in-context learning for out-of-distribution video-language tasks with large multimodal models. We address the challenge of multimodal in-context examples exceeding the token length limit of LMMs by employing similarity-based demonstration selection and confidence-based iteration. Extensive experimental results highlight the effectiveness of our method for OOD videos.
\ours is training-free, enabling rapid adaptation to novel tasks while offering a more efficient and feasible alternative to naive in-context learning at inference time.

\section{Acknowledgments}
This work was supported by Institute of Information \& communications Technology Planning \& Evaluation (IITP) grant funded by the Korea government (MSIT) (No. RS-2022-00187238, Development of Large Korean Language Model Technology for Efficient Pre-training) and (No. RS-2019-II190075, Artificial Intelligence Graduate School Program(KAIST)).
{
    \small
    \bibliographystyle{ieeenat_fullname}
    \bibliography{main}
}
\clearpage
\appendix

\noindent\textbf{\large{Appendices}}
\section{Discussion on Hyperparameter Choice}
\label{sec:rationale}
\begin{table}[th]
    \newcommand{\markstar}{\makebox[0pt][l]{*}}
    \newcommand{\grey}[1]{\cellcolor[HTML]{eeeeee}{#1}}
    \centering
    \resizebox{0.95\linewidth}{!}{%
    \begin{tabular}{lcccc}
         \toprule
         & \makecell{Animal\\Kingdom} & PitVQA & UCF-Crime & \grey{Avg.} \\
         \midrule
         Zero-shot & 68.0 & 6.7 & 39.3 & \grey{38.0} \\
         \midrule
         $c_{th}$ = 0.1 & 69.4 & 53.6 & 50.3 & \grey{57.7} \\
         $c_{th}$ = 0.3 & 69.5 & 58.5 & 50.6 & \grey{59.5} \\
         $c_{th}$ = 0.5 & 70.7 & 61.3\markstar & 53.3\markstar & \grey{61.8} \\
         $c_{th}$ = 0.7 & 72.3\markstar & \textbf{61.6} & 52.7 & \grey{62.2} \\
         $c_{th}$ = 0.9 & \textbf{72.6} & 61.5 & \textbf{53.6} & \grey{\textbf{62.6}} \\
         \bottomrule
    \end{tabular}%
    }
    \caption{\textbf{Ablation on confidence threshold $c_{th}$.} The values used for the main table are marked with *.}
    \label{tab:ablation_confidence_threshold}
\end{table}

\noindent
We compare the results on varying confidence threshold $c_{th}$ in \cref{tab:ablation_confidence_threshold}. While the accuracy generally increases with $c_{th}=0.9$, it also increases the cost of the entire process by performing more iterations per query on average. Therefore, we choose $c_{th}=0.5$ and $0.7$ for the best trade-off between cost and accuracy.

\section{Details on Datasets}
\label{sec:appendix_B}
\paragraph{Animal Kingdom}
We use the Animal Kingdom dataset~\citep{ng_animal_2022} for our multiple choice question answering task. This dataset includes videos of animals with action labels such as \textit{Yawning} and \textit{Struggling}, covering 140 unique classes. While it was originally built for action recognition tasks, we modified its format to suit a multiple-choice QA task by pairing one true action label with four randomly chosen alternative labels. The dataset provides 24,004 labeled training examples and 6,096 test examples.
\paragraph{Sports-QA}
We employ the Sports-QA~\citep{li_sports-qa_2024} dataset for open-ended question answering task, which is designed for sports video question answering. This dataset includes various sports, such as basketball, football, and gymnastics, and features diverse question types like descriptions, timelines, causalities, and hypothetical scenarios. The dataset includes 56,385 training examples and 18,718 test examples.
\paragraph{PitVQA}
We also use PitVQA~\citep{he2024pitvqaimagegroundedtextembedding}, a dataset designed for VQA in endonasal pituitary surgery videos that requires specific medical knowledge, for the open-ended question answering task. PitVQA provides question-answer annotations at the frame level. For our experiments, we process a sequence of 10 consecutive frames as the video input, with question-answer pairs drawn from the middle, fifth frame. The dataset includes 75,010 training examples and 10,832 test examples.
\paragraph{UCF-Crime}
UCF-Crime~\citep{sultani_real-world_2019}, which classifies the type of crime in security camera footage into 13 categories, is used for video classification task. We include all crime categories in the prompt, guiding the model to select the appropriate crime class for the given video. The dataset also includes normal event videos as challenging negative examples. The official split of UCF-Crime provides four different train and test splits, with each split consisting of 532 training samples and 168 test samples. The result is reported as the average performance across the test sets of all four splits.
\paragraph{Drive\&Act}
The Drive\&Act dataset~\citep{martin_driveact_2019} is utilized for video classification tasks. This offers comprehensive labels for driver behaviors inside vehicles, including action segmentation information captured in Kinect-IR videos. We extract each segment from the video and ask the model to recognize the action. The official split of Drive\&Act provides three different train and test splits. Each split consists of around 2,000 labeled training examples and around 600 test examples. The result is reported as the average performance across the test sets of all three splits.
\paragraph{CapERA}
For the video captioning task, we evaluate models on the CapERA dataset~\citep{bashmal_capera_2023}, which is specifically curated for describing scenes captured from an aerial perspective. CapERA provides concise captions for a range of scenarios viewed from above, including concerts, harvesting, and car racing, and consists of 1,473 labeled examples for training and 1,391 for testing.

\section{Proof of Asymptotic Model Accuracy}
\label{sec:appendix_C}
\begin{proposition*}[Asymptotic Model Accuracy]
    Let $a(n)$ be the expected accuracy of \ours with a maximum of $n$ confidence-based iterations. Then,
    \begin{align*}
        \lim_{n \rightarrow \infty} a(n) = \frac{1}{1+\frac{\mathrm{FPR}}{\mathrm{TPR}} \cdot \frac{1-p_c}{p_c}},
    \end{align*}
    where $\mathrm{TPR}$ and $\mathrm{FPR}$ stand for the true positive rate (i.e., recall) and the false positive rate of the confidence estimation method, respectively.
\end{proposition*}

\begin{proof}

    At each iteration, there are three possibilities: \begin{itemize}
        \item The model returns a correct response and is estimated to be confident, with probability $p_c \cdot \mathrm{TPR}$.
        \item The model returns an incorrect response, but is estimated to be confident, with a probability of $(1-p_c)\cdot \mathrm{FPR}$.
        \item The model returns a response, and is estimated to be unconfident, occurring with probability $p_u := 1 - (p_c \cdot \mathrm{TPR} + (1-p_c) \cdot \mathrm{FPR})$.
    \end{itemize}

    For the first two cases, the loop terminates and returns a response, whereas in the third case, the loop continues with a new iteration. Let $c(n)$ represent the probability that the loop ends by correctly returning a response (first case) on the $n$-th iteration, and $l(n)$ represent the probability that the loop is still ongoing (third case) after $n$ iterations.

    The expected accuracy $a(n)$ after $n$ iterations is the sum of the probabilities of ending with a correct response up to the $n$-th iteration, plus the probability of continuing after the $(n-1)$-th iteration, weighted by the probability of a correct response in the next iteration $p_c$:
    \begin{equation}
        a(n) = \sum_{i=1}^{n} c(i) + l(n-1) \cdot p_c.
        \label{eq:a(n)}
    \end{equation}
    
    The probability of continuing after the $n$-th iteration is $l(n) = l(n-1) \cdot p_u$, with $l(0) = 1$, leading to $l(n) = p_u^n$ by recursion. And the probability $c(n)$ of ending at the $n$-th iteration with a correct and confident response is:
    \begin{equation}
        c(n) = l(n-1) \cdot (p_c \cdot \mathrm{TPR}) = p_u^{n-1} \cdot (p_c \cdot \mathrm{TPR}),
        \label{eq:c(n)}
    \end{equation}
    where $p_c \cdot \mathrm{TPR}$ accounts for the likelihood that a response is classified as confident ($p_c$) and is also correct ($\mathrm{TPR}$).
    Therefore, we have:
    \begin{equation}
        a(n) = p_c \cdot \mathrm{TPR} \cdot \frac{1 - p_u^n}{1 - p_u} + p_u^{n-1} \cdot p_c.
    \end{equation}
    Since $0<p_u<1$, as $n \to \infty$, 
    \begin{align}
        \lim_{n \to \infty} a(n) & = \frac{p_c \cdot \mathrm{TPR}}{1 - p_u} \notag \\
        & = \frac{1}{1 + \frac{\mathrm{FPR}}{\mathrm{TPR}} \cdot \frac{1 - p_c}{p_c}}.
    \end{align}
    
\end{proof}

\section{Additional Results}


\subsection{Additional Discussion on Main Results}
\label{sec:additional_discussion}
For LoRA fine-tuning, we use a rank of 32 and train the model for 1 epoch on Animal Kingdom and PitVQA, 5 epochs on UCF-Crime, and 2 epochs on CapERA. In \cref{tab:main}, \ours outperforms the LoRA fine-tuned model on all datasets except CapERA, showing that in-context examples are more effective than training in OOD video QA when domain knowledge requires extensive data and training.

Interestingly, the LLaVA-Video-72B model underperforms compared to LLaVA-Video-7B model notably in video classification and captioning. For captioning, this is because 72B model often generates excessively long outputs filled with irrelevant details. In video classification, we suspect the limited capacity of 7B model may act as a form of regularization, helping it generalize better on OOD data, but this needs further investigation.

In addition, \ours outperforms \textsc{SimRankVote} in \cref{tab:main}, highlighting the benefits of using confidence-based aggregation instead of majority voting. \ours also achieves better results than \textsc{SimRankOnce}, showing that using more examples leads to better performance. 
Lastly, \textsc{SimRankVote} outperforms \textsc{RandExVote}, demonstrating the effectiveness of selecting similar examples based on video and text features.

\subsection{Additional Qualitative Results}
\label{sec:additional_qualitative}
In the following pages, we present qualitative results of \ours for each dataset. For each iteration, we use two examples with maximum of 4 iterations, and the outputs of the model are presented together with confidence scores.

\section{Limitation}
While \ours delivers remarkable performance, it does have some limitations. First, \ours requires more time compared to single-step in-context learning because it performs multiple rounds of inference. This additional computation may make it less suitable for applications that demand low latency, such as real-time video analysis. However, \ours mitigates this issue by using early termination when the model confidence in its output is sufficiently high, which significantly reduces computation time. It is also much faster than training a model from scratch.

Second, \ours relies on having an example pool to select reference examples from. We demonstrate its effectiveness on the UCF-Crime dataset, which contains only 532 training samples, showing that \ours can perform well even with a relatively small example pool. However, we have not tested its performance with extremely small datasets. Considering the challenges of generating out-of-distribution video data, exploring the effectiveness of \ours with very limited examples is an important direction for future research.

\vspace*{\fill} 
\begin{center}
\begin{figure*}[h]
\resizebox{\linewidth}{!}{
\begin{tcolorbox}[
    enhanced,                 
    colframe=black,           
    colback=white,            
    boxrule=1.5pt,            
    width=\textwidth,         
    before skip=1em,          
    after skip=1em,           
    overlay={%
        \node[anchor=north west, fill=black, font=\large, text=white, inner sep=2mm, 
        xshift=4mm, yshift=4mm, rounded corners=1mm] 
        at (frame.north west) {Multiple Choice QA: Animal Kingdom};
    }
]
\vspace{1em} 

\leftline{\textbf{$\blacktriangleright$ Iteration 1}}
\vspace{1em}
\renewcommand{\arraystretch}{1.5}
\begin{tabular}{m{0.4\textwidth}m{9cm}}
\includegraphics[width=0.4\textwidth,height=1.7cm]{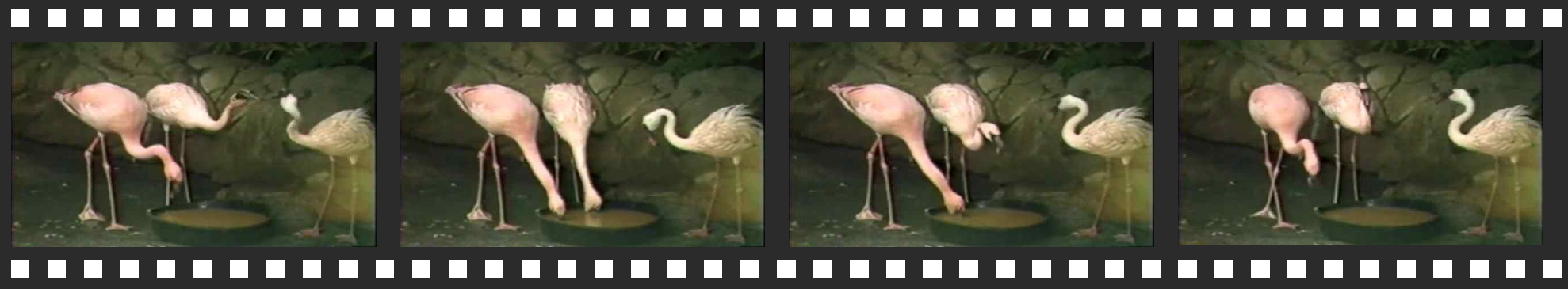} &
\textbf{Example 1:} What action is the animal doing in the video? Answer with the option's letter from the given choices directly. Options: (A) Dancing On Water (B) Urinating (C) Eating (D) Sleeping in its nest (E) Sharing Food \newline
The answer is (C) \\
\midrule
\includegraphics[width=0.4\textwidth,height=1.7cm]{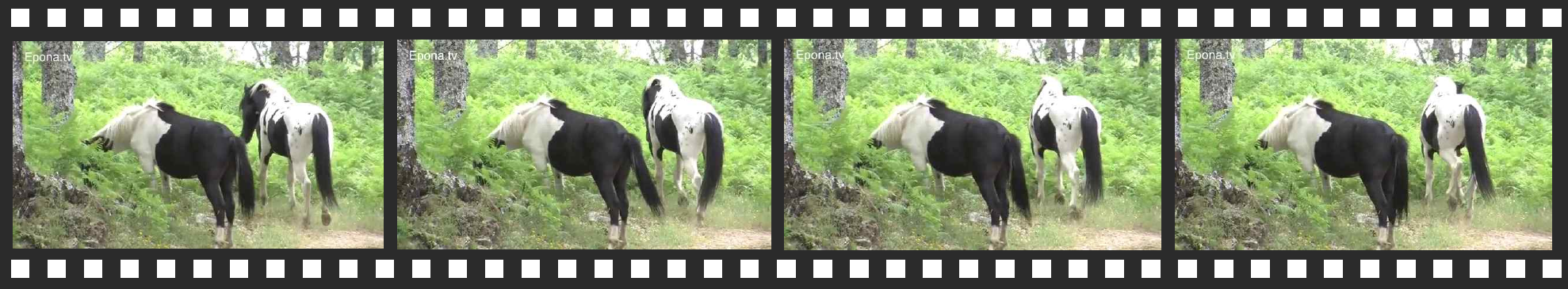} &
\textbf{Example 2:} What action is the animal doing in the video? Answer with the option's letter from the given choices directly. Options: (A) Biting (B) Climbing (C) Dancing (D) Drinking (E) Being Dragged \newline
The answer is (D) \\
\midrule
\includegraphics[width=0.4\textwidth,height=1.7cm]{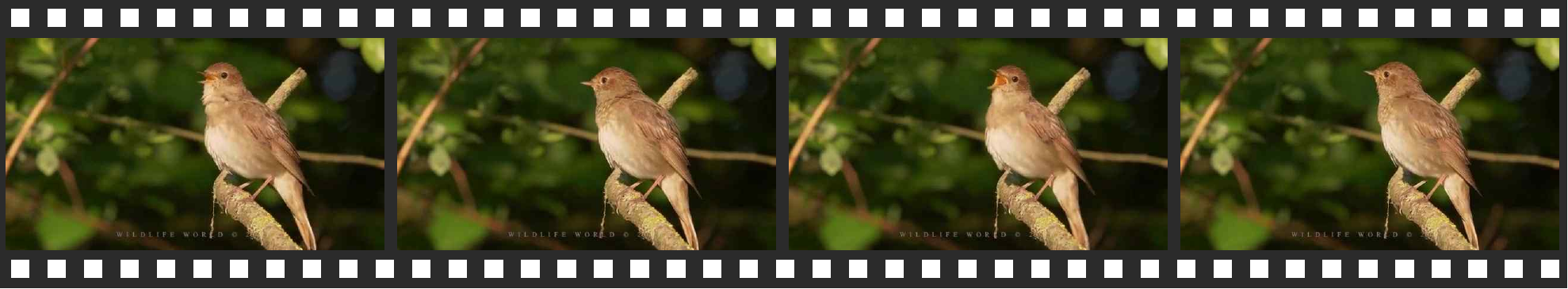} &
\textbf{User:} What action is the animal doing in the video? Answer with the option's letter from the given choices directly.  Options: (A) Trapped (B) Attending (C) Walking On Water (D) Sharing Food (E) Showing Affection \newline
\textbf{LLaVA-Video:} The answer is (D) \textcolor{Red}{(Wrong, Confidence 0.409)}\\
\end{tabular}

\leftline{\textbf{$\blacktriangleright$ Iteration 2}}
\vspace{1em}
\begin{tabular}{m{0.4\textwidth}m{9cm}}
\includegraphics[width=0.4\textwidth,height=1.7cm]{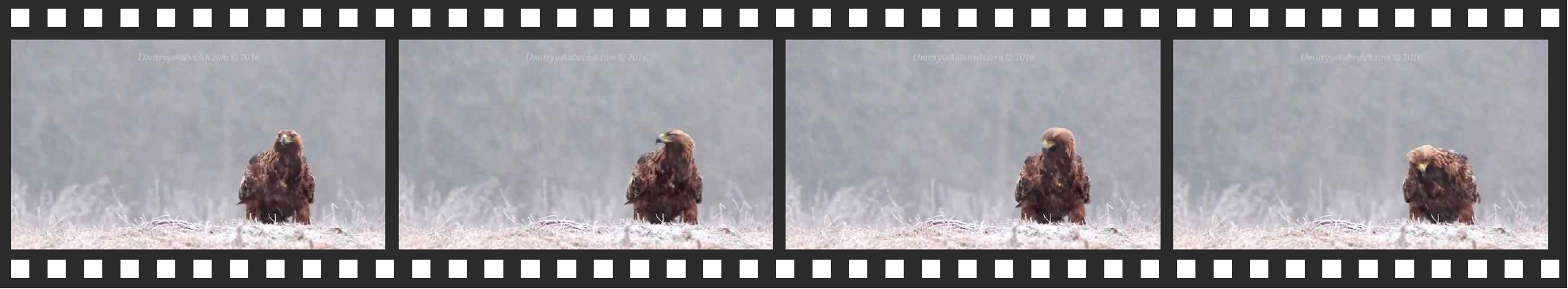} &
\textbf{Example 3:} What action is the animal doing in the video? Answer with the option's letter from the given choices directly. Options: (A) Chasing (B) Rubbing its head (C) Moving (D) Dying (E) Doing somersault \newline
The answer is (C) \\
\midrule
\includegraphics[width=0.4\textwidth,height=1.7cm]{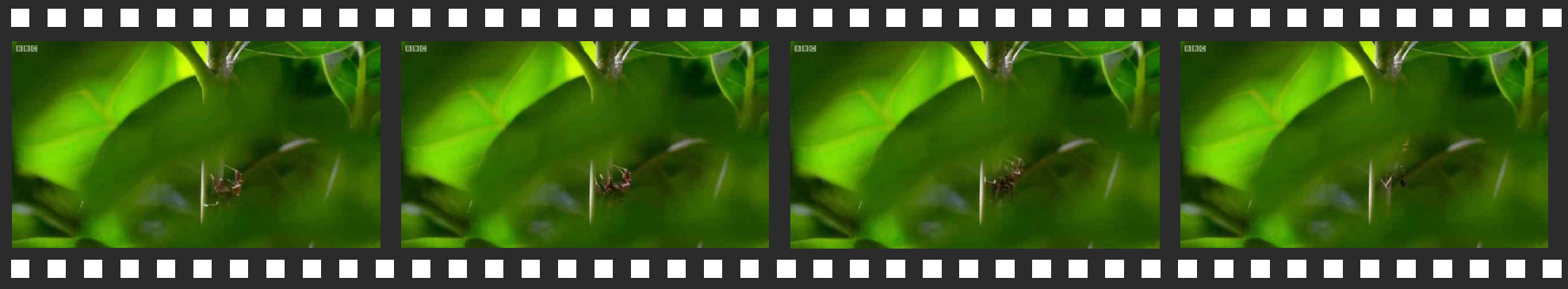} &
\textbf{Example 4:} What action is the animal doing in the video? Answer with the option's letter from the given choices directly. Options: (A) Puffing its throat (B) Hatching (C) Doing A Back Kick (D) Attending (E) Lying Down \newline
The answer is (D) \\
\midrule
\includegraphics[width=0.4\textwidth,height=1.7cm]{qualitative_results/animal/animal_test.pdf} &
\textbf{User:} What action is the animal doing in the video? Answer with the option's letter from the given choices directly.  Options: (A) Trapped (B) Attending (C) Walking On Water (D) Sharing Food (E) Showing Affection \newline
\textbf{LLaVA-Video:} The answer is (B) \textcolor{ForestGreen}{(Correct, Confidence 0.886)}\\
\end{tabular}

\vspace{1mm}
\end{tcolorbox}
}
\caption{Qualitative result on the Animal Kingdom dataset.}
\end{figure*}
\end{center}
\vspace*{\fill} 

\vspace*{\fill} 
\begin{center}
\begin{figure*}[h]
\resizebox{\linewidth}{!}{
\begin{tcolorbox}[
    enhanced,                 
    colframe=black,           
    colback=white,            
    boxrule=1.5pt,            
    width=\textwidth,         
    before skip=1em,          
    after skip=1em,           
    overlay={%
        \node[anchor=north west, fill=black, font=\large, text=white, inner sep=2mm, 
        xshift=4mm, yshift=4mm, rounded corners=1mm] 
        at (frame.north west) {Open-ended QA: Sports-QA};
    }
]
\vspace{1em} 

\leftline{\textbf{$\blacktriangleright$ Iteration 1}}
\vspace{2mm}
\begin{tabular}{m{0.4\textwidth}m{9cm}}
\includegraphics[width=0.4\textwidth]{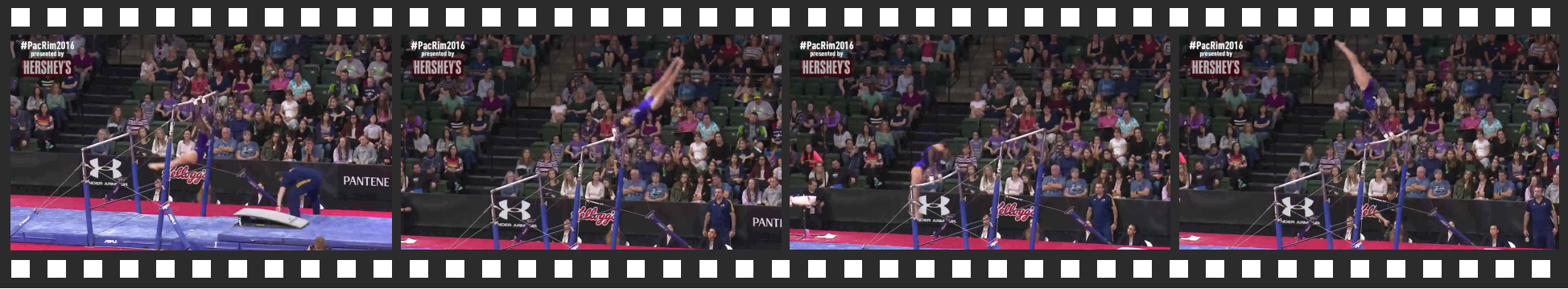} &
\textbf{Example 1:} What do the players perform after performing transition flight from low bar to high bar? \newline
Giant circle backward with 1 turn to handstand \\
\midrule
\includegraphics[width=0.4\textwidth]{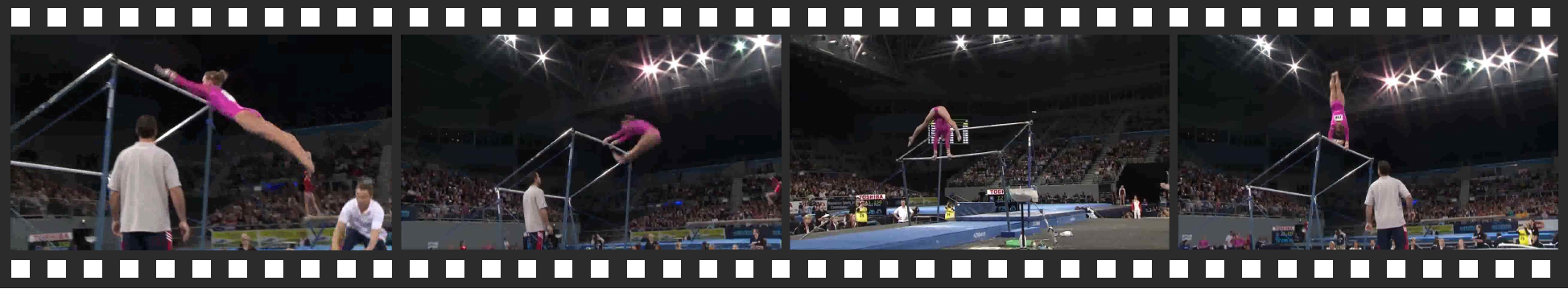} &
\textbf{Example 2:} What do the players perform after performing transition flight from low bar to high bar? \newline
Transition flight from high bar to low bar \\
\midrule
\includegraphics[width=0.4\textwidth]{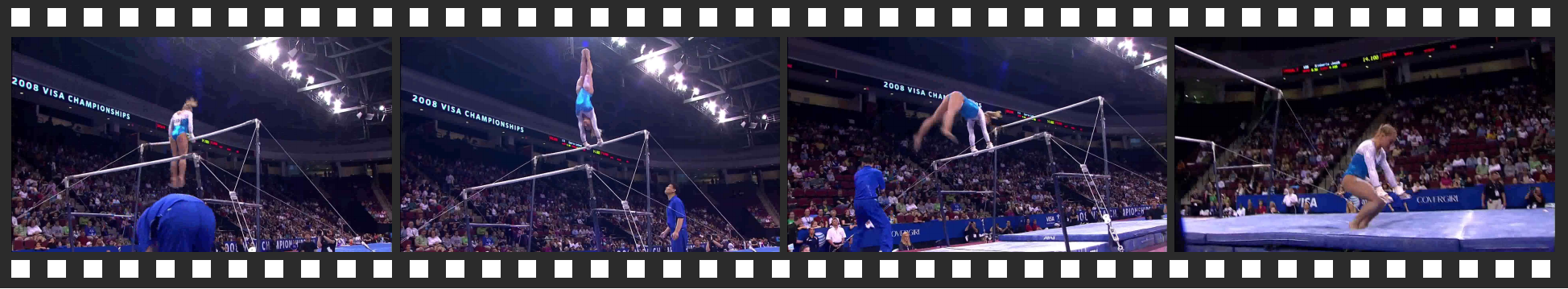} &
\textbf{User:} What do the players perform after performing transition flight from low bar to high bar? \newline
\textbf{LLaVA-Video:} Giant circle backward with 1 turn to handstand \textcolor{Red}{(Correct, Confidence 0.409)}\\
\end{tabular}

\leftline{\textbf{$\blacktriangleright$ Iteration 2}}
\vspace{2mm}
\begin{tabular}{m{0.4\textwidth}m{9cm}}
\includegraphics[width=0.4\textwidth]{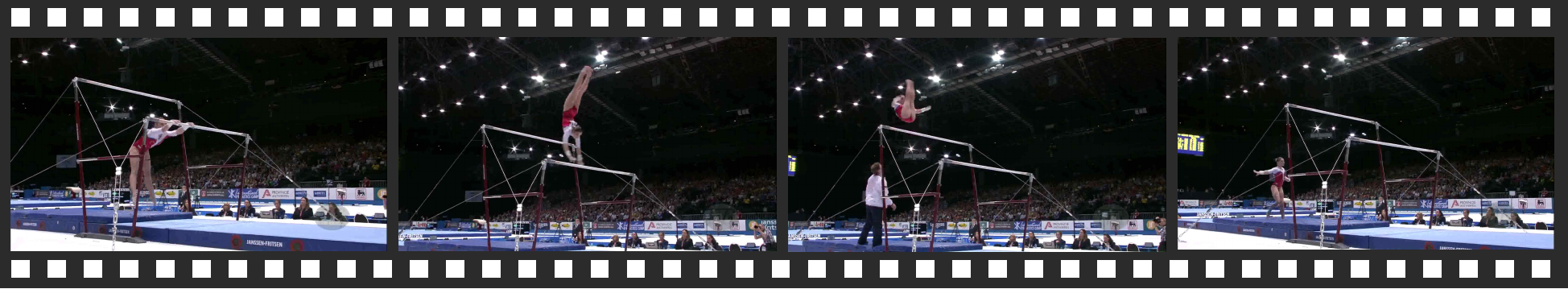} &
\textbf{Example 3:} What do the players perform after performing transition flight from low bar to high bar? \newline
Transition flight from high bar to low bar \\
\midrule
\includegraphics[width=0.4\textwidth]{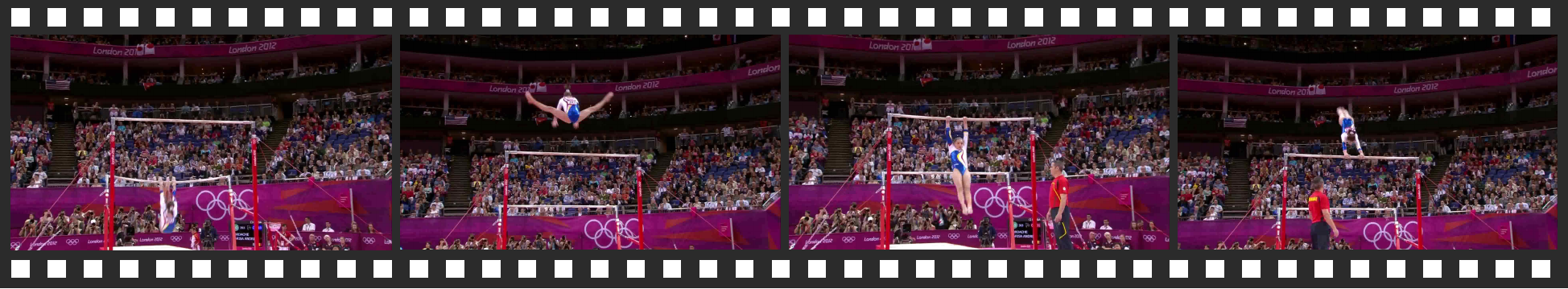} &
\textbf{Example 4:} What do the players perform after performing transition flight from low bar to high bar? \newline
Giant circle backward with 0.5 turn to handstand \\
\midrule
\includegraphics[width=0.4\textwidth]{qualitative_results/sportsqa/sportsqa_test.pdf} &
\textbf{User:} What do the players perform after performing transition flight from low bar to high bar? \newline
\textbf{LLaVA-Video:} Giant circle backward with 0.5 turn to handstand \textcolor{Red}{(Wrong, Confidence 0.309)}\\
\end{tabular}

\leftline{\textbf{$\blacktriangleright$ Iteration 3}}
\vspace{2mm}
\begin{tabular}{m{0.4\textwidth}m{9cm}}
\includegraphics[width=0.4\textwidth]{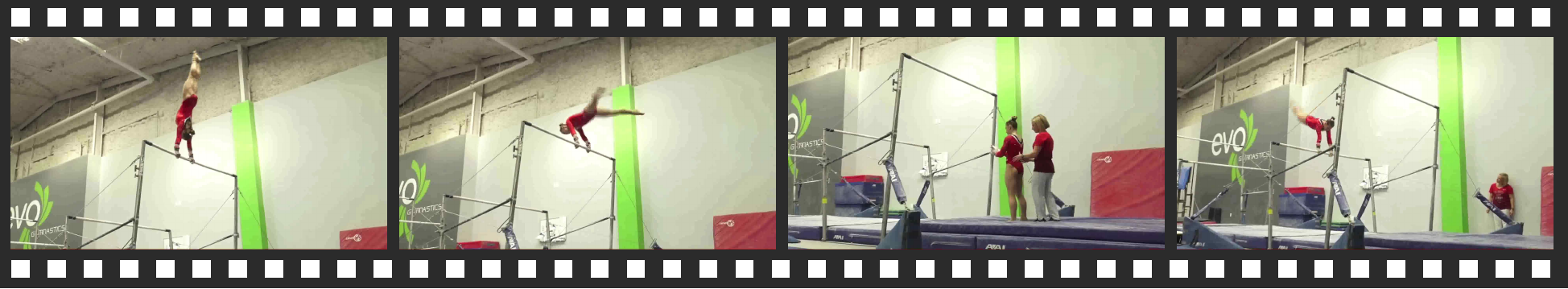} &
\textbf{Example 5:} What do the players perform after performing transition flight from low bar to high bar? \newline
Giant circle backward \\
\midrule
\includegraphics[width=0.4\textwidth]{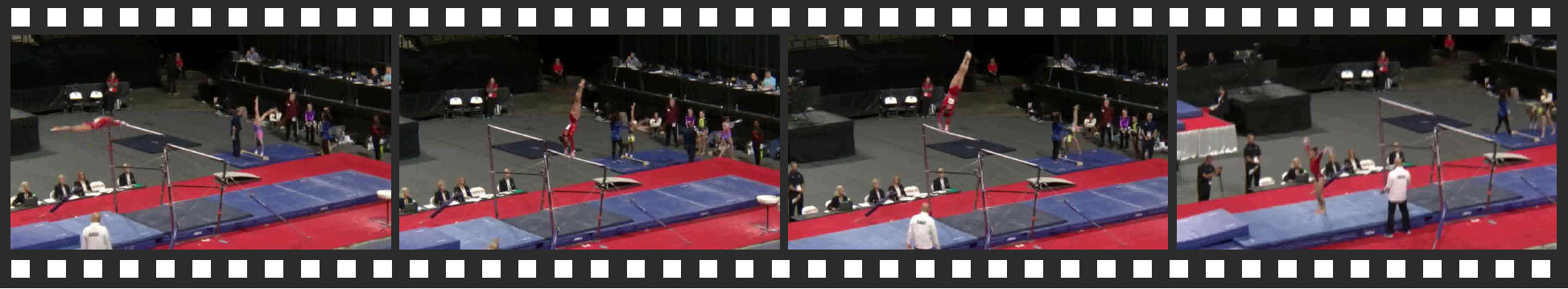} &
\textbf{Example 6:} What do the players perform after performing transition flight from low bar to high bar? \newline
Pike sole circle backward with 0.5 turn to handstand \\
\midrule
\includegraphics[width=0.4\textwidth]{qualitative_results/sportsqa/sportsqa_test.pdf} &
\textbf{User:} What do the players perform after performing transition flight from low bar to high bar? \newline
\textbf{LLaVA-Video:} Pike sole circle backward with 0.5 turn to handstand \textcolor{Red}{(Wrong, Confidence 0.303)}\\
\end{tabular}

\leftline{\textbf{$\blacktriangleright$ Iteration 4}}
\vspace{2mm}
\begin{tabular}{m{0.4\textwidth}m{9cm}}
\includegraphics[width=0.4\textwidth]{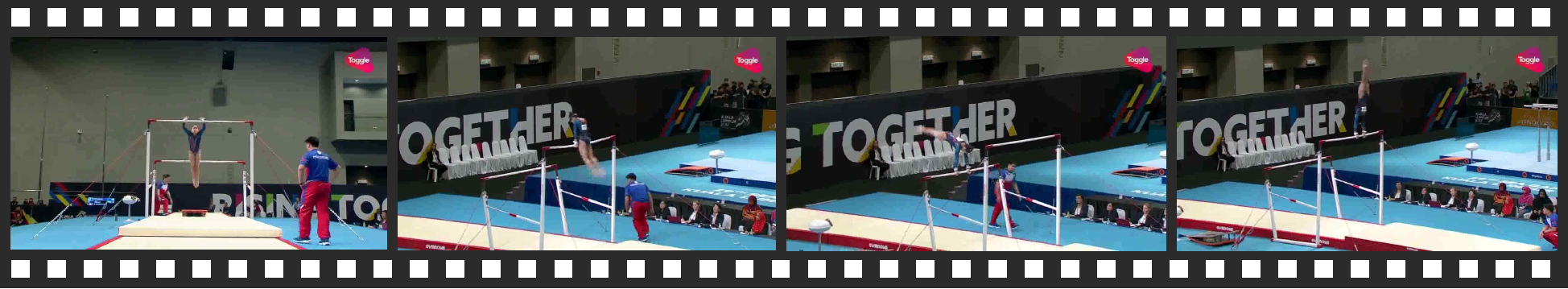} &
\textbf{Example 7:} What do the players perform after performing transition flight from low bar to high bar? \newline
Giant circle backward with 1 turn to handstand \\
\midrule
\includegraphics[width=0.4\textwidth]{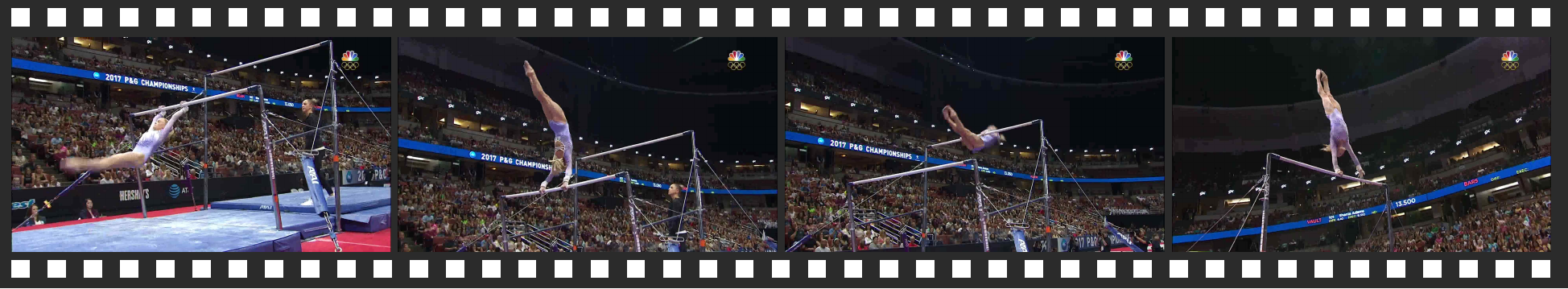} &
\textbf{Example 8:} What do the players perform after performing transition flight from low bar to high bar? \newline
Giant circle backward with 0.5 turn to handstand \\
\midrule
\includegraphics[width=0.4\textwidth]{qualitative_results/sportsqa/sportsqa_test.pdf} &
\textbf{User:} What do the players perform after performing transition flight from low bar to high bar? \newline
\textbf{LLaVA-Video:} Giant circle backward with 1 turn to handstand \textcolor{ForestGreen}{(Correct, Confidence 0.586)}\\
\end{tabular}
\vspace{1mm}
\end{tcolorbox}
}
\caption{Qualitative result on the Sports-QA dataset.}
\end{figure*}
\end{center}
\vspace*{\fill} 

\vspace*{\fill} 
\begin{center}
\begin{figure*}[h]
\resizebox{\linewidth}{!}{
\begin{tcolorbox}[
    enhanced,                 
    colframe=black,           
    colback=white,            
    boxrule=1.5pt,            
    width=\textwidth,         
    before skip=1em,          
    after skip=1em,           
    overlay={%
        \node[anchor=north west, fill=black, font=\large, text=white, inner sep=2mm, 
        xshift=4mm, yshift=4mm, rounded corners=1mm] 
        at (frame.north west) {Open-ended QA: PitVQA};
    }
]
\vspace{1em} 

\leftline{\textbf{$\blacktriangleright$ Iteration 1}}
\vspace{2mm}
\begin{tabular}{m{0.4\textwidth}m{9cm}}
\includegraphics[width=0.4\textwidth,height=1.7cm]{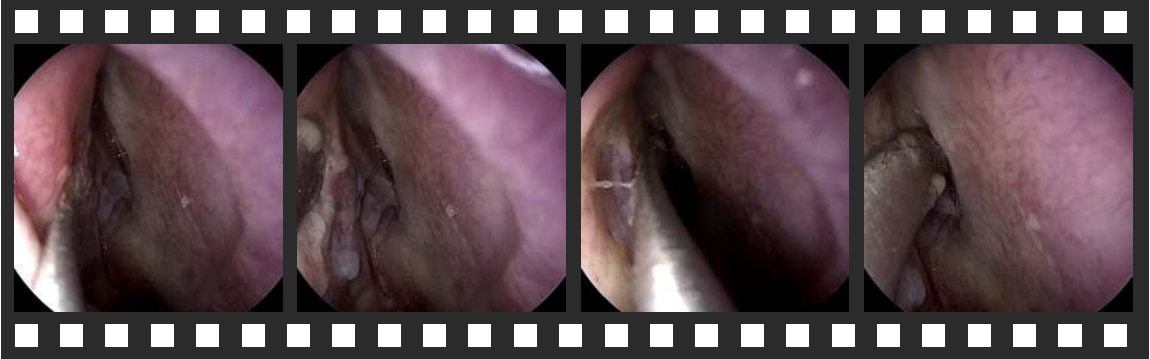} &
\textbf{Example 1:} Where is the surgical instrument freer elevator tip located in the middle of the video? \newline
Top-left \\
\midrule
\includegraphics[width=0.4\textwidth,height=1.7cm]{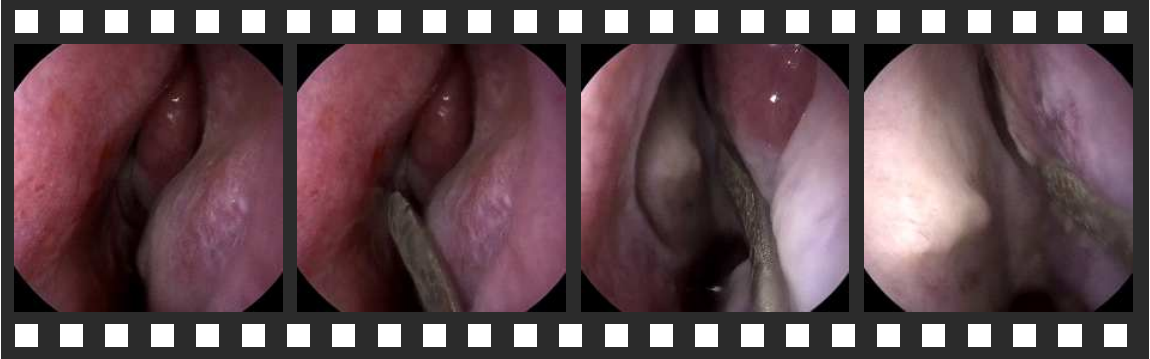} &
\textbf{Example 2:} Where is the surgical instrument freer elevator tip located in the middle of the video? \newline
Centre \\
\midrule
\includegraphics[width=0.4\textwidth,height=1.7cm]{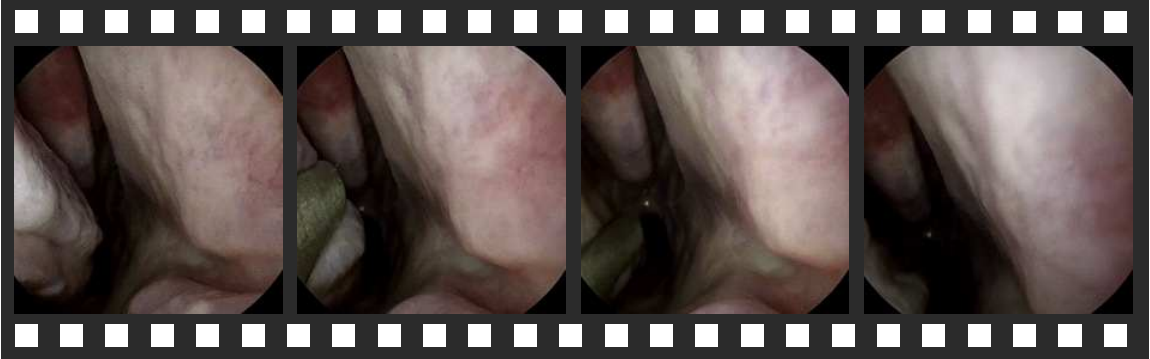} &
\textbf{User:} Where is the surgical instrument freer elevator tip located in the middle of the video? \newline
\textbf{LLaVA-Video:} Bottom-right \textcolor{Red}{(Wrong, Confidence 0.167)}\\
\vspace{1mm}
\end{tabular}

\leftline{\textbf{$\blacktriangleright$ Iteration 2}}
\vspace{2mm}
\begin{tabular}{m{0.4\textwidth}m{9cm}}
\includegraphics[width=0.4\textwidth,height=1.7cm]{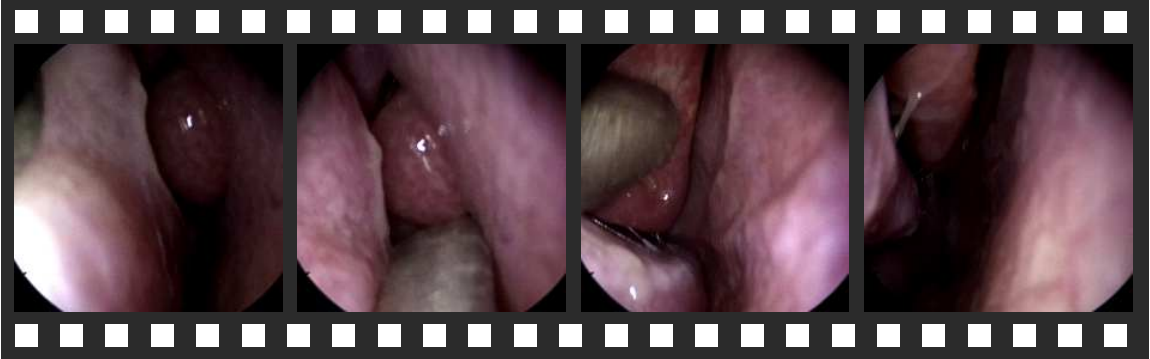} &
\textbf{Example 3:} Where is the surgical instrument freer elevator tip located in the middle of the video? \newline
Top-left \\
\midrule
\includegraphics[width=0.4\textwidth,height=1.7cm]{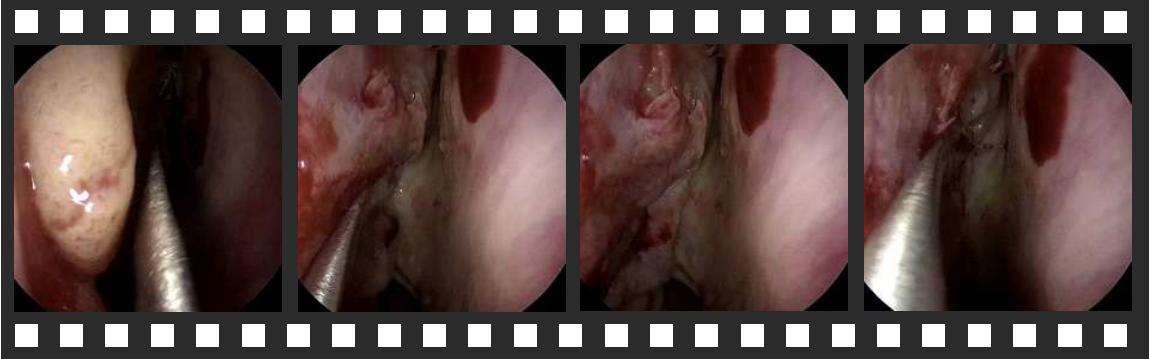} &
\textbf{Example 4:} Where is the surgical instrument freer elevator tip located in the middle of the video? \newline
Centre \\
\midrule
\includegraphics[width=0.4\textwidth,height=1.7cm]{qualitative_results/pitvqa/pitvqa_test.pdf} &
\textbf{User:} Where is the surgical instrument freer elevator tip located in the middle of the video? \newline
\textbf{LLaVA-Video:} Top \textcolor{Red}{(Wrong, Confidence 0.204)}\\
\vspace{1mm}
\end{tabular}

\leftline{\textbf{$\blacktriangleright$ Iteration 3}}
\vspace{2mm}
\begin{tabular}{m{0.4\textwidth}m{9cm}}
\includegraphics[width=0.4\textwidth,height=1.7cm]{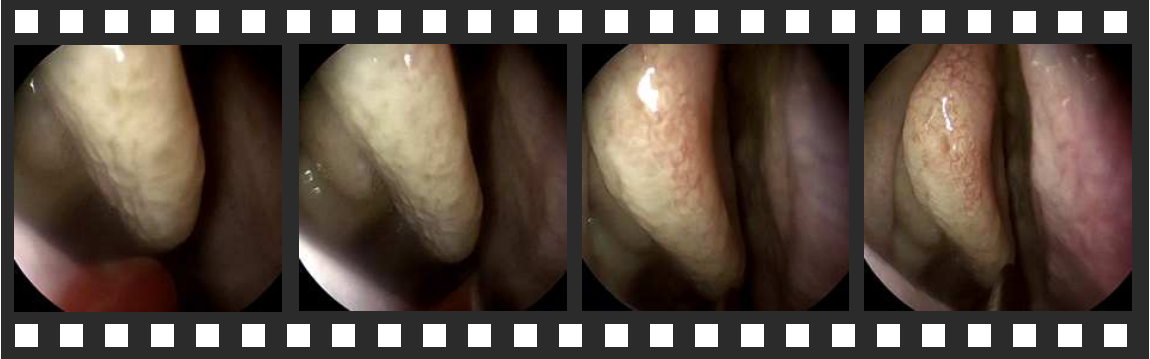} &
\textbf{Example 5:} Where is the surgical instrument freer elevator tip located in the middle of the video? \newline
Bottom-left \\
\midrule
\includegraphics[width=0.4\textwidth,height=1.7cm]{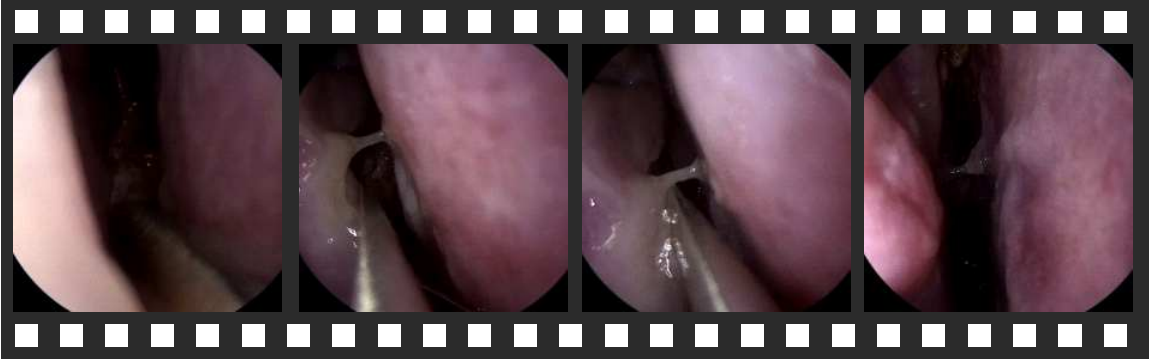} &
\textbf{Example 6:} Where is the surgical instrument freer elevator tip located in the middle of the video? \newline
Bottom-left \\
\midrule
\includegraphics[width=0.4\textwidth,height=1.7cm]{qualitative_results/pitvqa/pitvqa_test.pdf} &
\textbf{User:} Where is the surgical instrument freer elevator tip located in the middle of the video? \newline
\textbf{LLaVA-Video:} Bottom-left \textcolor{ForestGreen}{(Correct, Confidence 0.809)}\\
\end{tabular}

\vspace{1mm}
\end{tcolorbox}
}
\caption{Qualitative result on the PitVQA dataset.}
\end{figure*}
\end{center}
\vspace*{\fill} 

\vspace*{\fill} 
\begin{center}
\begin{figure*}[h]
\resizebox{\linewidth}{!}{
\begin{tcolorbox}[
    enhanced,                 
    colframe=black,           
    colback=white,            
    boxrule=1.5pt,            
    width=\textwidth,         
    before skip=1em,          
    after skip=1em,           
    overlay={%
        \node[anchor=north west, fill=black, font=\large, text=white, inner sep=2mm, 
        xshift=4mm, yshift=4mm, rounded corners=1mm] 
        at (frame.north west) {Video Classification: UCF-Crime};
    }
]
\vspace{1em} 

\leftline{\textbf{$\blacktriangleright$ Iteration 1}}
\vspace{1em}
\begin{tabular}{m{0.4\textwidth}m{9cm}}
\includegraphics[width=0.4\textwidth,height=1.7cm]{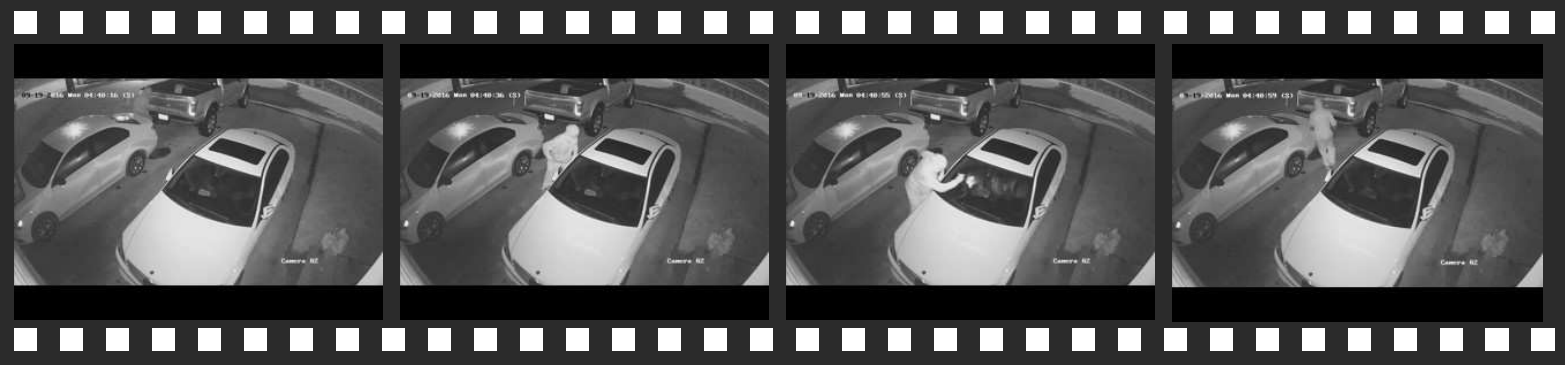} &
\textbf{Example 1:} Classify the following video into one of the following categories: \{14 categories\} \newline
Stealing \\
\midrule
\includegraphics[width=0.4\textwidth,height=1.7cm]{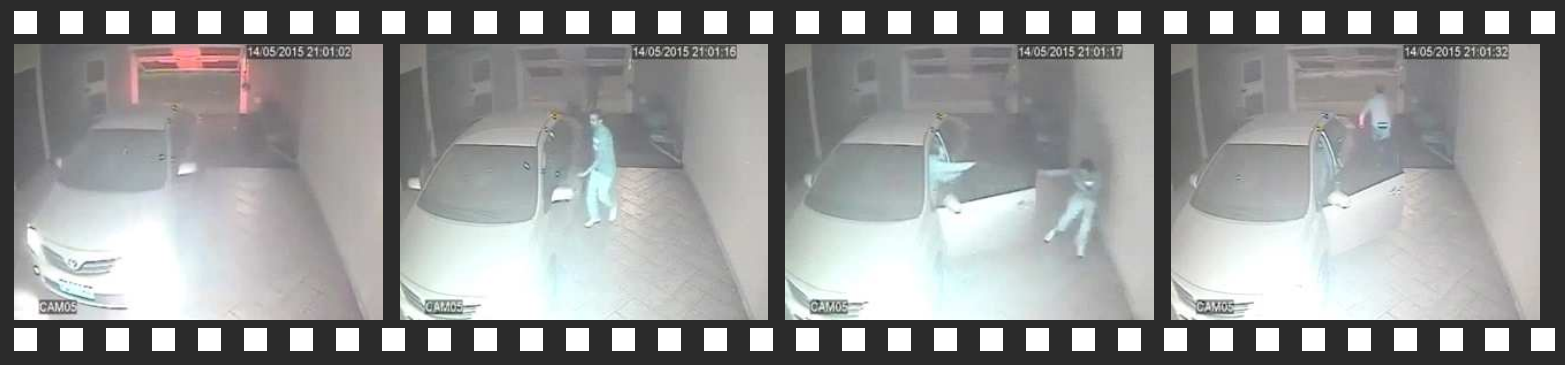} &
\textbf{Example 2:} Classify the following video into one of the following categories: \{14 categories\} \newline
Shooting \\
\midrule
\includegraphics[width=0.4\textwidth,height=1.7cm]{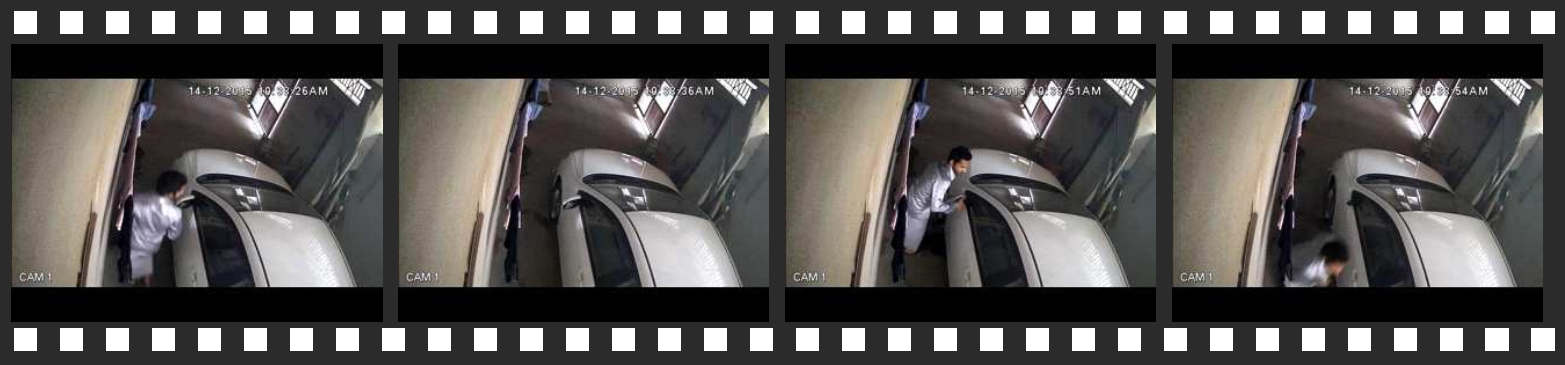} &
\textbf{User:} Classify the following video into one of the following categories: \{14 categories\} \newline
\textbf{LLaVA-Video:} Stealing \textcolor{Red}{(Correct, Confidence 0.484)}\\
\vspace{1mm}
\end{tabular}

\leftline{\textbf{$\blacktriangleright$ Iteration 2}}
\vspace{1em}
\begin{tabular}{m{0.4\textwidth}m{9cm}}
\includegraphics[width=0.4\textwidth,height=1.7cm]{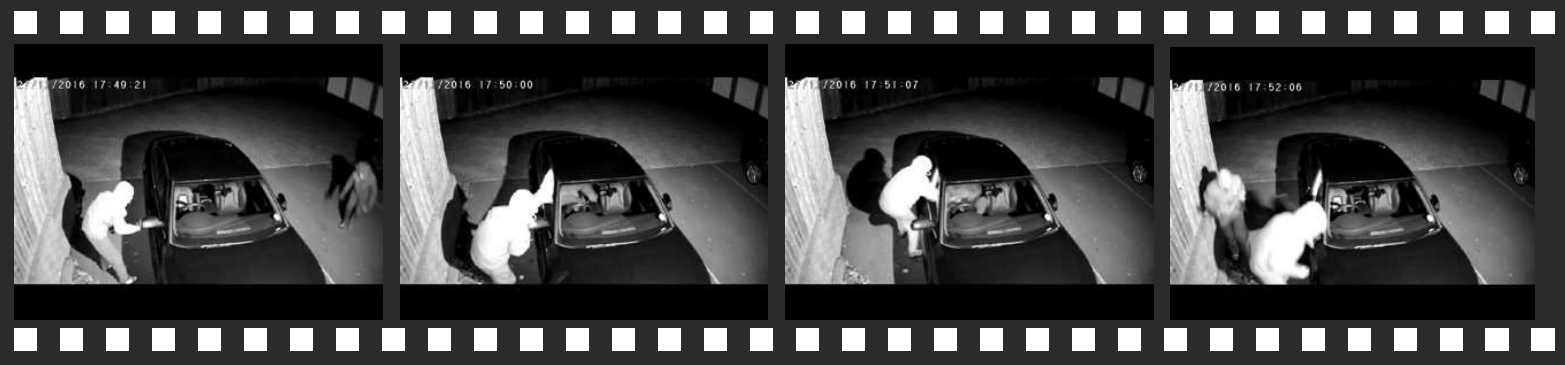} &
\textbf{Example 3:} Classify the following video into one of the following categories: \{14 categories\} \newline
Stealing \\
\midrule
\includegraphics[width=0.4\textwidth,height=1.7cm]{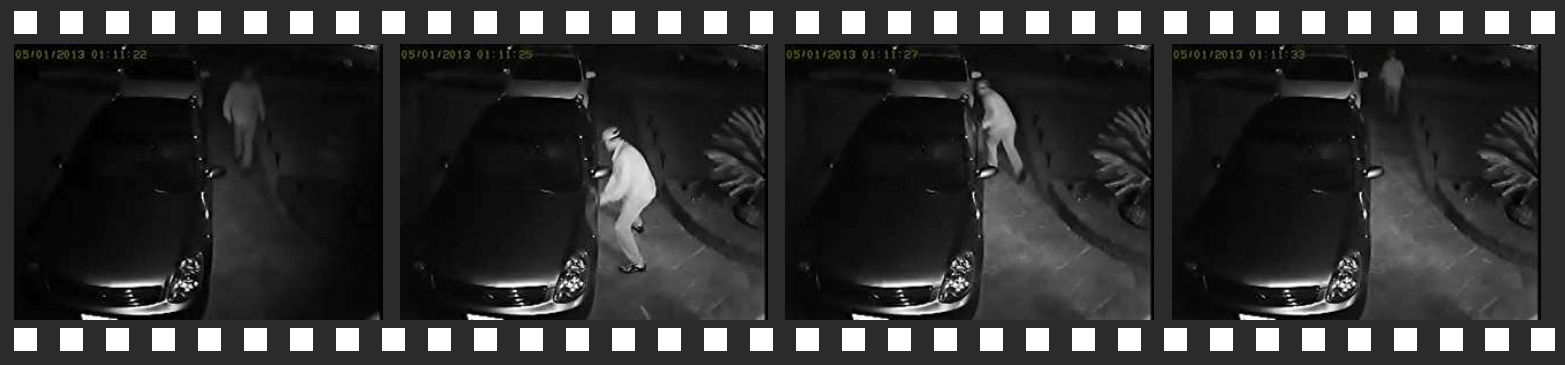} &
\textbf{Example 4:} Classify the following video into one of the following categories: \{14 categories\} \newline
Vandalism \\
\midrule
\includegraphics[width=0.4\textwidth,height=1.7cm]{qualitative_results/crime/crime_test.pdf} &
\textbf{User:} Classify the following video into one of the following categories: \{14 categories\} \newline
\textbf{LLaVA-Video:} Stealing \textcolor{ForestGreen}{(Correct, Confidence 0.591)}\\
\end{tabular}

\vspace{1mm}
\end{tcolorbox}
}
\caption{Qualitative result on the UCF-Crime dataset.}
\end{figure*}
\end{center}
\vspace*{\fill} 

\vspace*{\fill} 
\begin{center}
\begin{figure*}[h]
\resizebox{\linewidth}{!}{
\begin{tcolorbox}[
    enhanced,                 
    colframe=black,           
    colback=white,            
    boxrule=1.5pt,            
    width=\textwidth,         
    before skip=1em,          
    after skip=1em,           
    overlay={%
        \node[anchor=north west, fill=black, font=\large, text=white, inner sep=2mm, 
        xshift=4mm, yshift=4mm, rounded corners=1mm] 
        at (frame.north west) {Video Classification: Drive\&Act};
    }
]
\vspace{1em} 

\leftline{\textbf{$\blacktriangleright$ Iteration 1}}
\vspace{1em}
\begin{tabular}{m{0.4\textwidth}m{9cm}}
\includegraphics[width=0.4\textwidth,height=1.7cm]{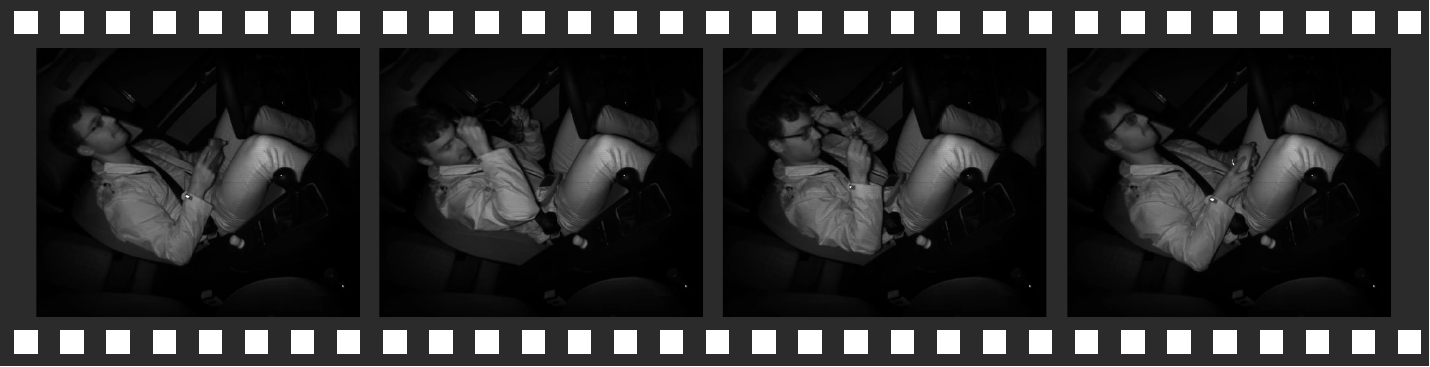} &
\textbf{Example 1:} Classify the following video into one of the following categories: \{34 categories\} \newline
Putting on sunglasses \\
\midrule
\includegraphics[width=0.4\textwidth,height=1.7cm]{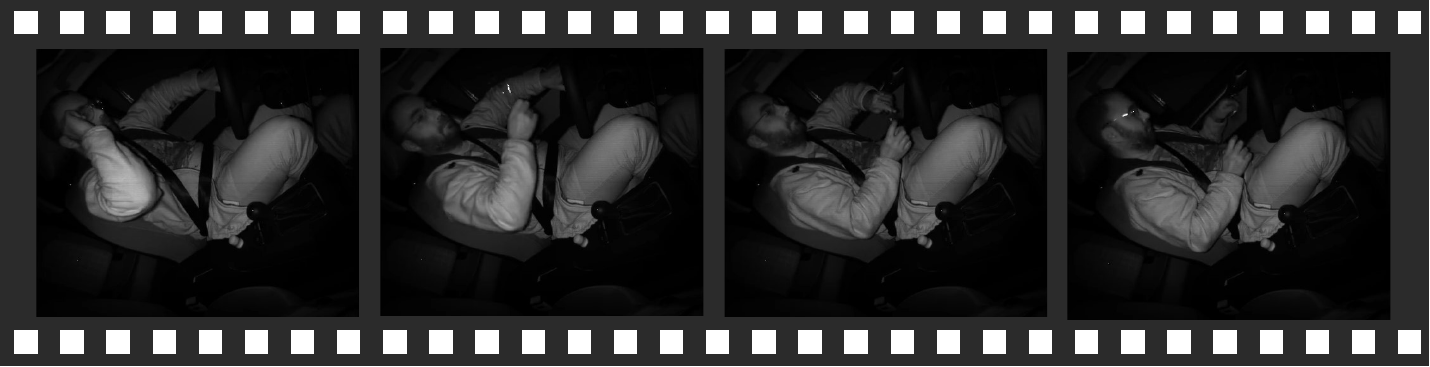} &
\textbf{Example 2:} Classify the following video into one of the following categories: \{34 categories\} \newline
Taking off sunglasses \\
\midrule
\includegraphics[width=0.4\textwidth,height=1.7cm]{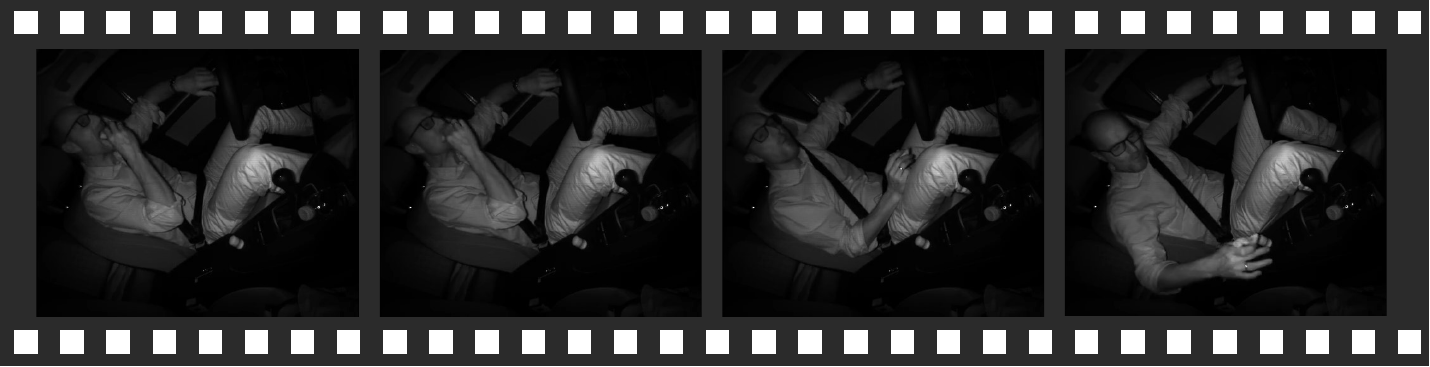} &
\textbf{User:} Classify the following video into one of the following categories: \{34 categories\} \newline
\textbf{LLaVA-Video:} Taking off sunglasses \newline \textcolor{Red}{(Wrong, Confidence 0.395)}\\
\vspace{1mm}
\end{tabular}

\leftline{\textbf{$\blacktriangleright$ Iteration 2}}
\vspace{1em}
\begin{tabular}{m{0.4\textwidth}m{9cm}}
\includegraphics[width=0.4\textwidth,height=1.7cm]{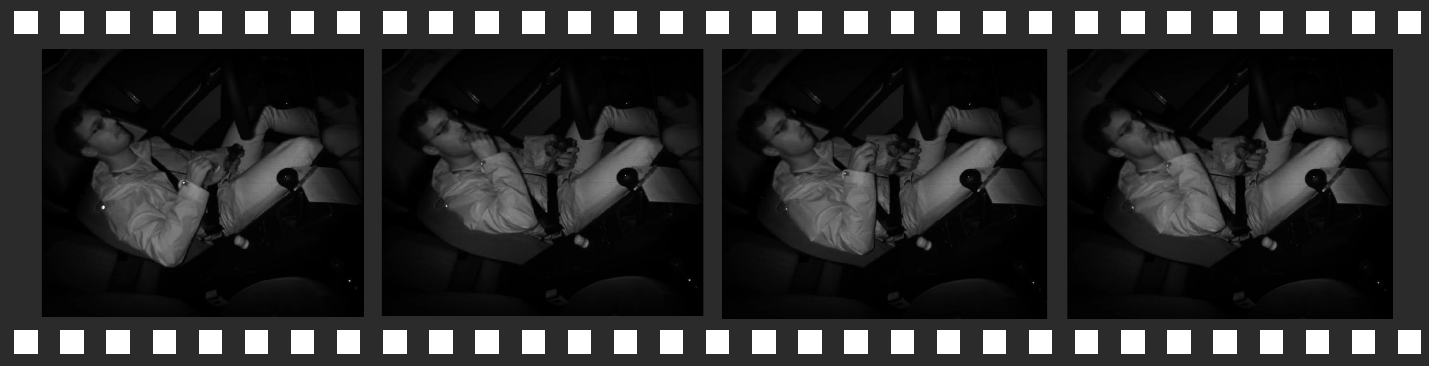} &
\textbf{Example 3:} Classify the following video into one of the following categories: \{34 categories\} \newline
Eating \\
\midrule
\includegraphics[width=0.4\textwidth,height=1.7cm]{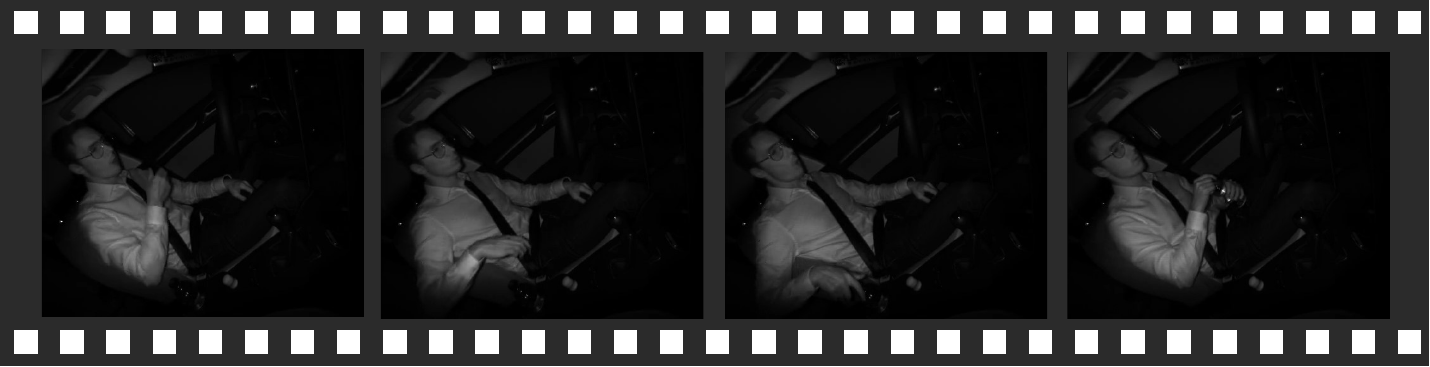} &
\textbf{Example 4:} Classify the following video into one of the following categories: \{34 categories\} \newline
Fetching an object \\
\midrule
\includegraphics[width=0.4\textwidth,height=1.7cm]{qualitative_results/drive/drive_test.pdf} &
\textbf{User:} Classify the following video into one of the following categories: \{34 categories\} \newline
\textbf{LLaVA-Video:} Eating \textcolor{Red}{(Correct, Confidence 0.354)}\\
\vspace{1mm}
\end{tabular}

\leftline{\textbf{$\blacktriangleright$ Iteration 3}}
\vspace{1em}
\begin{tabular}{m{0.4\textwidth}m{9cm}}
\includegraphics[width=0.4\textwidth,height=1.7cm]{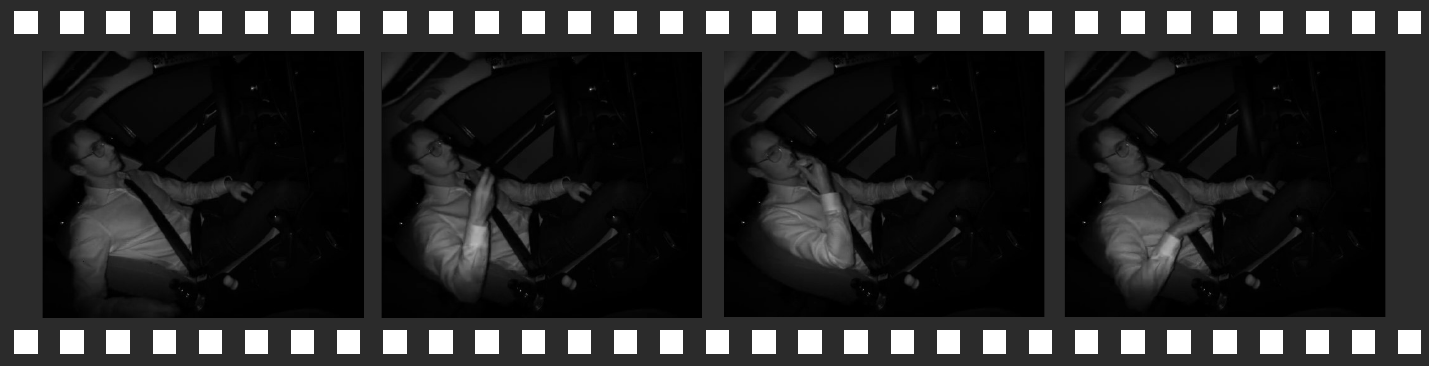} &
\textbf{Example 5:} Classify the following video into one of the following categories: \{34 categories\} \newline
Eating \\
\midrule
\includegraphics[width=0.4\textwidth,height=1.7cm]{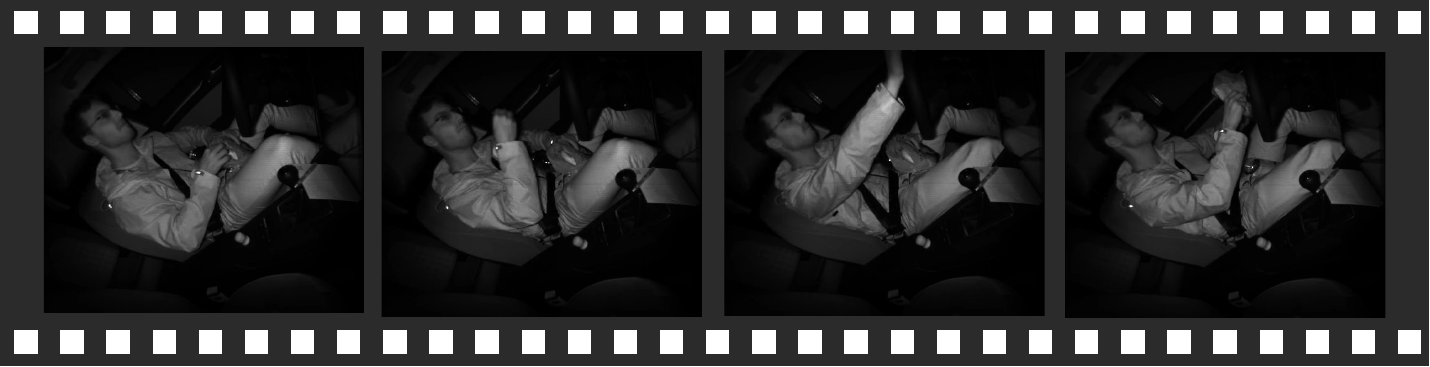} &
\textbf{Example 6:} Classify the following video into one of the following categories: \{34 categories\} \newline
Eating \\
\midrule
\includegraphics[width=0.4\textwidth,height=1.7cm]{qualitative_results/drive/drive_test.pdf} &
\textbf{User:} Classify the following video into one of the following categories: \{34 categories\} \newline
\textbf{LLaVA-Video:} Eating \textcolor{ForestGreen}{(Correct, Confidence 0.818)}\\
\end{tabular}

\vspace{1mm}
\end{tcolorbox}
}
\caption{Qualitative result on the Drive\&Act dataset.}
\end{figure*}
\end{center}
\vspace*{\fill} 

\vspace*{\fill} 
\begin{center}
\begin{figure*}[h]
\resizebox{\linewidth}{!}{
\begin{tcolorbox}[
    enhanced,                 
    colframe=black,           
    colback=white,            
    boxrule=1.5pt,            
    width=\textwidth,         
    before skip=1em,          
    after skip=1em,           
    overlay={%
        \node[anchor=north west, fill=black, font=\large, text=white, inner sep=2mm, 
        xshift=4mm, yshift=4mm, rounded corners=1mm] 
        at (frame.north west) {Video Captioning: CapERA};
    }
]
\vspace{1em} 

\leftline{\textbf{$\blacktriangleright$ Iteration 1}}
\vspace{2mm}
\begin{tabular}{m{0.4\textwidth}m{9cm}}
\includegraphics[width=0.4\textwidth,height=1.3cm]{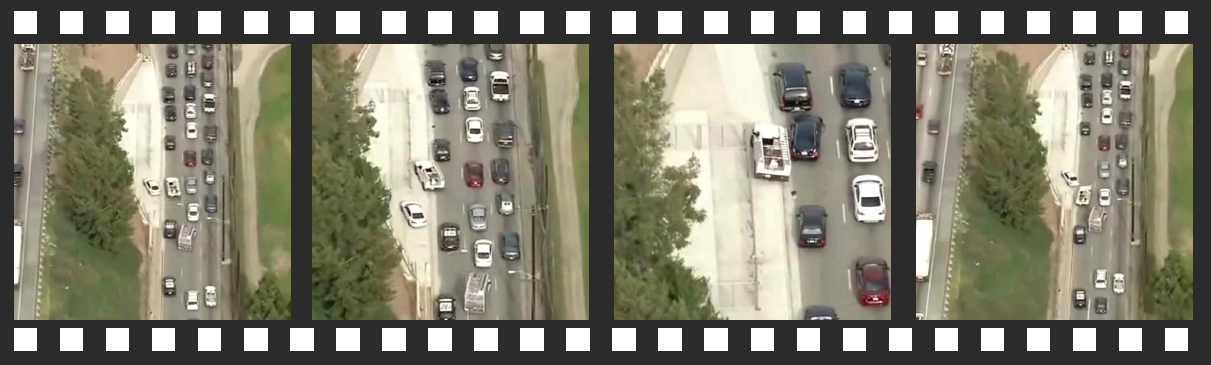} &
\textbf{Example 1:} Provide a concise depiction of this video. \newline
Two police cars were chasing a white car down a busy street while the white car was walking on the sidewalk. \\
\midrule
\includegraphics[width=0.4\textwidth,height=1.3cm]{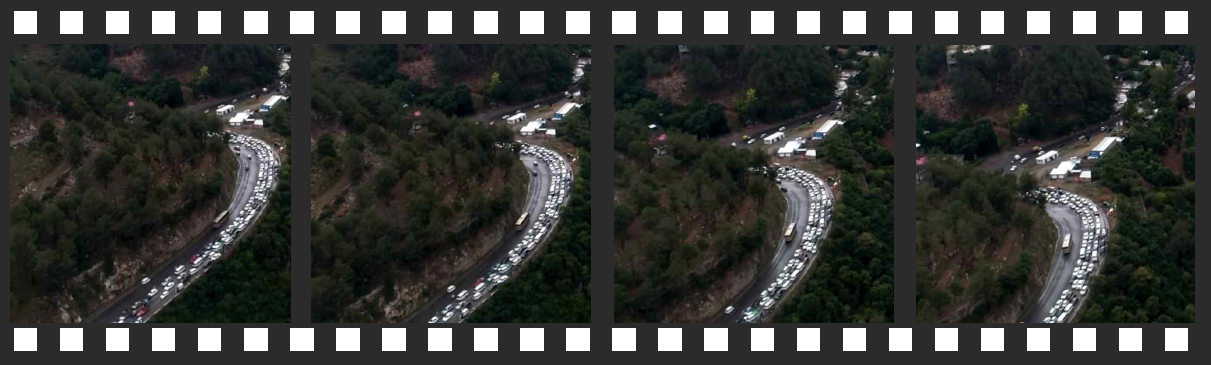} &
\textbf{Example 2:} Provide a concise depiction of this video. \newline
The winding mountain road is crowded with cars and surrounded by trees. \\
\midrule
\includegraphics[width=0.4\textwidth,height=1.3cm]{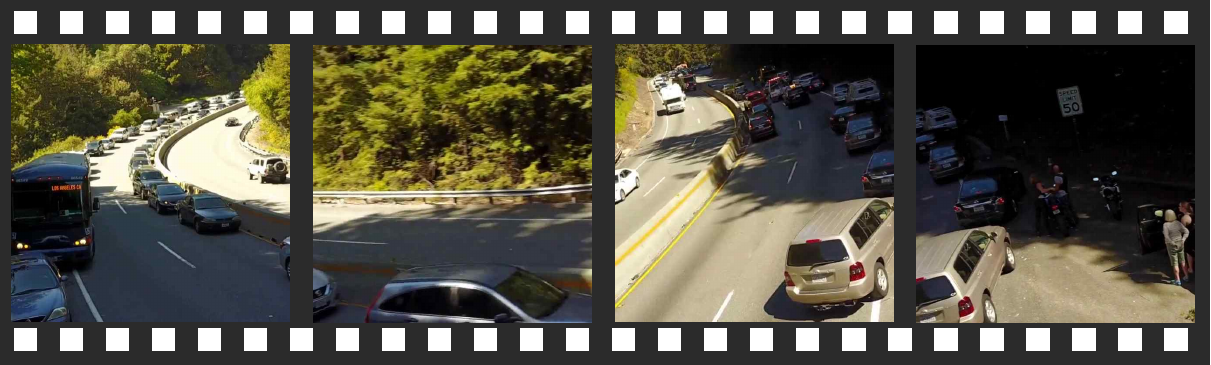} &
\textbf{User:} Provide a concise depiction of this video. \newline
\textbf{LLaVA-Video:} The traffic is moving slowly on the road. \newline \textcolor{Red}{(Confidence 0.114)}\\
\end{tabular}

\leftline{\textbf{$\blacktriangleright$ Iteration 2}}
\vspace{2mm}
\begin{tabular}{m{0.4\textwidth}m{9cm}}
\includegraphics[width=0.4\textwidth,height=1.3cm]{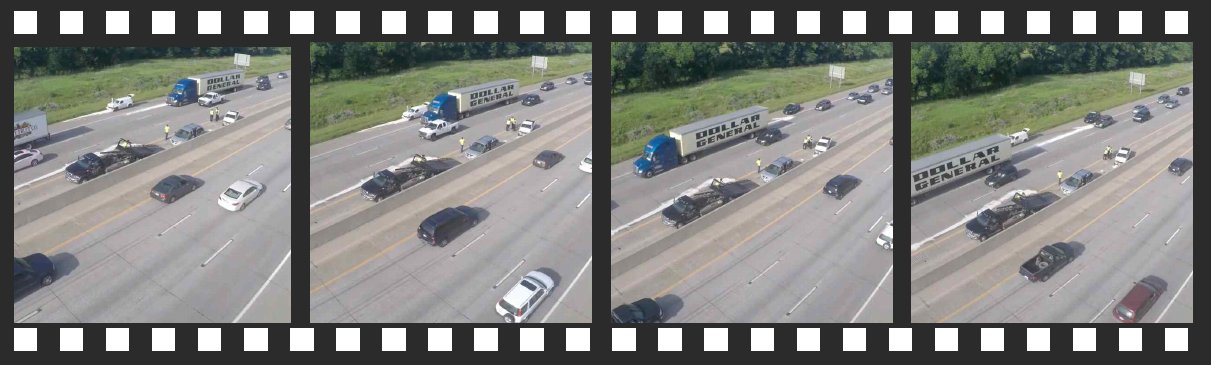} &
\textbf{Example 3:} Provide a concise depiction of this video. \newline
Three cars collided on a treelined road and policemen explored the accident. \\
\midrule
\includegraphics[width=0.4\textwidth,height=1.3cm]{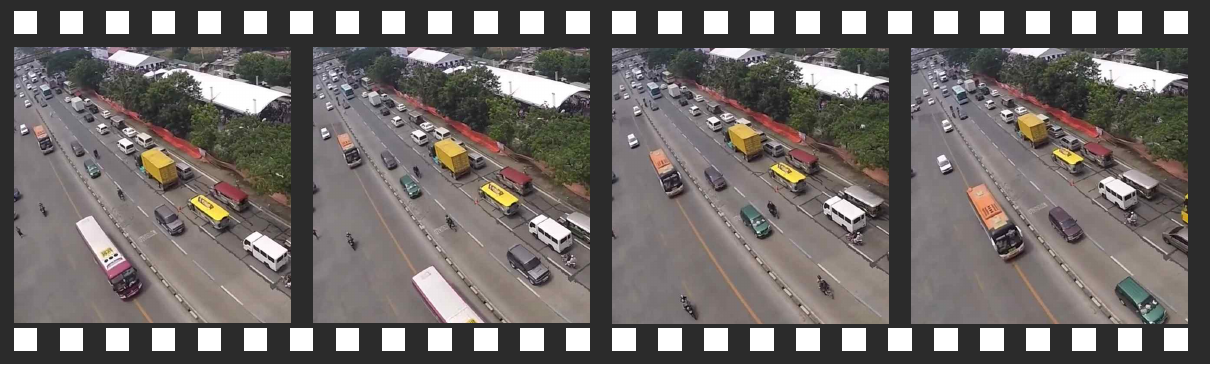} &
\textbf{Example 4:} Provide a concise depiction of this video. \newline
The road is crowded with cars and surrounded by small buildings and trees. \\
\midrule
\includegraphics[width=0.4\textwidth,height=1.3cm]{qualitative_results/capera/capera_test.pdf} &
\textbf{User:} Provide a concise depiction of this video. \newline
\textbf{LLaVA-Video:} Cars are driving on a road surrounded by trees and greenery. \textcolor{Red}{(Confidence 0.090)}\\
\end{tabular}

\leftline{\textbf{$\blacktriangleright$ Iteration 3}}
\vspace{2mm}
\begin{tabular}{m{0.4\textwidth}m{9cm}}
\includegraphics[width=0.4\textwidth,height=1.3cm]{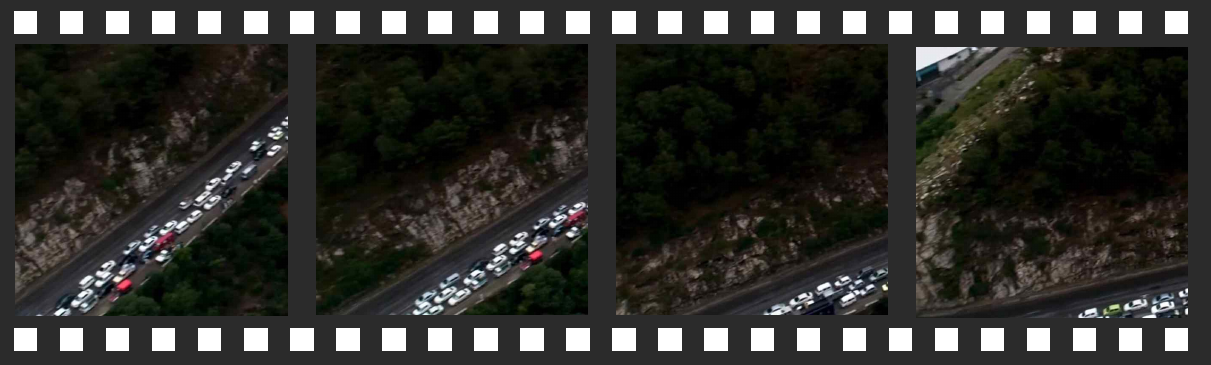} &
\textbf{Example 5:} Provide a concise depiction of this video. \newline
The winding mountain road is crowded with cars and surrounded by trees. \\
\midrule
\includegraphics[width=0.4\textwidth,height=1.3cm]{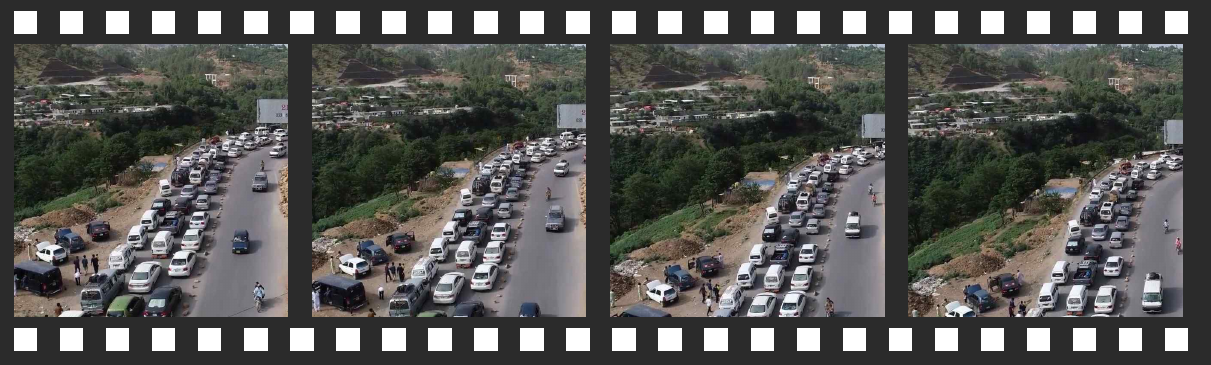} &
\textbf{Example 6:} Provide a concise depiction of this video. \newline
The winding mountain road is crowded with cars and surrounded by trees. \\
\midrule
\includegraphics[width=0.4\textwidth,height=1.3cm]{qualitative_results/capera/capera_test.pdf} &
\textbf{User:} Provide a concise depiction of this video. \newline
\textbf{LLaVA-Video:} The winding mountain road is crowded with cars and surrounded by trees. \textcolor{ForestGreen}{(Confidence 0.271)}\\
\end{tabular}

\leftline{\textbf{$\blacktriangleright$ Iteration 4}}
\vspace{2mm}
\begin{tabular}{m{0.4\textwidth}m{9cm}}
\includegraphics[width=0.4\textwidth,height=1.3cm]{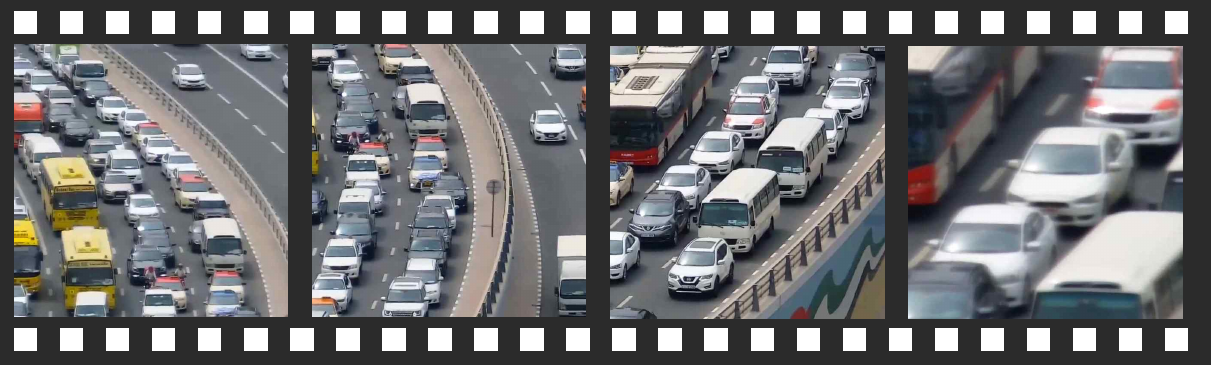} &
\textbf{Example 7:} Provide a concise depiction of this video. \newline
Crowded road with cars. \\
\midrule
\includegraphics[width=0.4\textwidth,height=1.3cm]{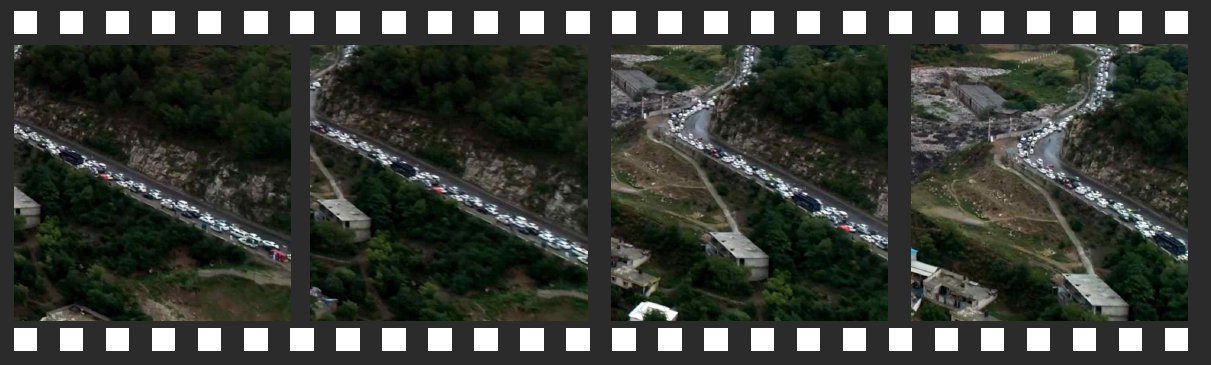} &
\textbf{Example 8:} Provide a concise depiction of this video. \newline
Winding mountain road crowded with cars and surrounded by trees. \\
\midrule
\includegraphics[width=0.4\textwidth,height=1.3cm]{qualitative_results/capera/capera_test.pdf} &
\textbf{User:} Provide a concise depiction of this video. \newline
\textbf{LLaVA-Video:} Traffic on a highway surrounded by trees. \newline \textcolor{Red}{(Confidence 0.179)}\\
\end{tabular}

\vspace{1mm}
\end{tcolorbox}
}
\caption{Qualitative result on the CapERA dataset.}
\end{figure*}
\end{center}
\vspace*{\fill} 
\clearpage

\end{document}